%% file: DCL-ENAS_main.tex
\documentclass[3p]{elsarticle}
\usepackage{algorithm}
\usepackage{algorithmic}

\usepackage{xcolor}
\usepackage{amsthm}   % 提供 theorem / proof 等环境
\usepackage{hyperref}
\usepackage{CJK}
\newcommand{\vect}[1]{\boldsymbol{#1}} % vector
\newcommand{\mat}[1]{\mathbf{#1}}      % matrix
\everymath{\displaystyle}              % block style for inline math

\newtheorem{proposition}{Proposition}[section]
\newtheorem{definition}{Definition}
\usepackage{times}
\usepackage{array,color}
\usepackage{booktabs}
\usepackage{multirow}
\usepackage{chngpage}
\usepackage{lscape}
\usepackage{amsfonts}
\usepackage{amsmath}
\usepackage{enumitem}
\usepackage{graphicx}
\usepackage{bm}
\usepackage{bbm}
\usepackage[caption=false,font=footnotesize]{subfig}
\usepackage{float}
\newcommand*{\Autoref}[1]{\textup{Section~\ref{#1}}}
\biboptions{numbers,sort&compress}

\begin{document}

\begin{frontmatter}

\title{Evolutionary Neural Architecture Search with Dual Contrastive Learning
%\tnoteref{mytitlenote}
}
\tnotetext[mytitlenote]{
    This work was supported by the Guangdong Natural Science Funds for Distinguished Young Scholars under Grant 2022B1515020049; by the National Natural Science Foundation of China under Grant 62276100; by the Guangdong Regional Joint Fund for Basic and Applied Research under Grant 2021B1515120078; by the National Research Foundation of Korea under Grant NRF2022H1D3A2A01093478; and by the TCL Young Scholars Program. 
}
\author[aff1]{Xian-Rong Zhang}
\author[aff1]{Yue-Jiao Gong\corref{mycorrespondingauthor}}
\ead{gongyuejiao@gmail.com}
\cortext[mycorrespondingauthor]{Corresponding author}
\author[aff1]{Wei-Neng Chen}
\author[aff2,aff3]{Jun Zhang}

\address[aff1]{School of Computer Science and Engineering, South China University of Technology, Guangzhou, China.}
\address[aff2]{College of Artificial Intelligence, Nankai University, Tianjin, China.}
\address[aff3]{Hanyang University, 15588 Ansan, South Korea.}

\begin{abstract}
Evolutionary Neural Architecture Search (ENAS) has gained attention for automatically designing neural network architectures. Recent studies use a neural predictor to guide the process, but the high computational costs of gathering training data---since each label requires fully training an architecture---make achieving a high-precision predictor with { limited compute budget (i.e., a capped number of fully trained architecture--label pairs)} crucial for ENAS success. This paper introduces ENAS with Dual Contrastive Learning (DCL-ENAS), a novel method that employs two stages of contrastive learning to train the neural predictor. In the first stage, contrastive self-supervised learning is used to learn meaningful representations from neural architectures without requiring labels. In the second stage, fine-tuning with contrastive learning is performed to accurately predict the relative performance of different architectures rather than their absolute performance, which is sufficient to guide the evolutionary search. Across NASBench-101 and NASBench-201, DCL-ENAS achieves the highest validation accuracy, surpassing the strongest published baselines by 0.05\% (ImageNet16-120) to 0.39\% (NASBench-101). On a real-world ECG arrhythmia classification task, DCL-ENAS improves performance by approximately 2.5 percentage points over a manually designed, non-NAS model obtained via random search, while requiring only 7.7 GPU-days.

\end{abstract}

\begin{keyword}
Evolutionary neural architecture search \sep predictor-assisted evolutionary algorithm \sep self-supervised learning \sep contrastive learning
\end{keyword}

\end{frontmatter}

% \linenumbers

\section{Introduction}
For decades, the promising architectures of deep neural network (DNN) have been manually designed by researchers with rich knowledge in neural networks and image processing. However, in practice, most users do not have such knowledge. In addition, DNN architectures are often problem-specific. If the distribution of data changes, the architecture must be redesigned accordingly~\cite{enas_survey}. Neural Architecture Search (NAS) is a technology that can automatically design architectures and is considered a promising method to address these challenges~\cite{nas_survey}. 
\begin{figure}[htbp]
    \centering
    \includegraphics[width=0.95\textwidth]{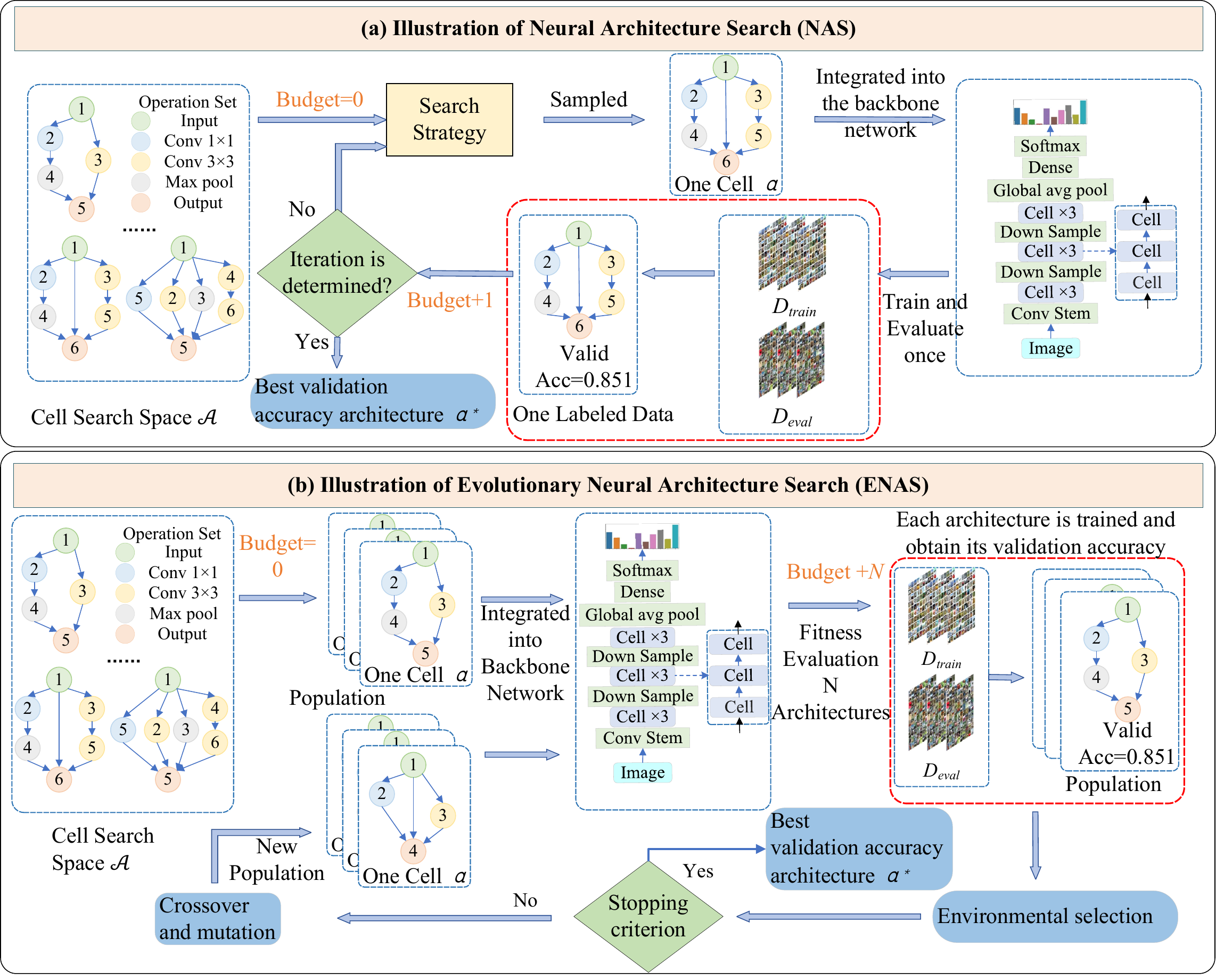}
    \caption{Illustration of NAS and ENAS using image classification as an example.}
    \label{fig:NAS}
\end{figure}

The goal of NAS is to automatically explore the optimal network architecture ${\alpha ^ * }$ within a predefined search space ${\cal A}$, which, when trained, achieves the maximum accuracy on the evaluation set ${D_{eval}}$. This process can be formalized as:
\begin{equation}
\left\{ {\begin{array}{*{20}{c}}
{{\alpha ^ * } = \mathop {\arg \max }\limits_{\alpha  \in {\cal A}} O(\alpha ,{w_{\alpha}^{*} },{D_{eval}})}\\
{\begin{array}{*{20}{c}}
{s.t.}&{{w_{\alpha}^{*} } = \mathop {\arg \max }\limits_{{w_\alpha }} O(\alpha ,{w_\alpha },{D_{train}})}
\end{array}}
\end{array}} \right.
\end{equation}
Here, ${O(\alpha ,{w_\alpha },{D})}$ refers to the objective function measuring the performance of the network architecture $\alpha$ under weights ${w_\alpha }$ on dataset $D$; $w_{\alpha}^{*}$ represents the weights of the optimal network architecture $\alpha ^ *$ achieving the best performance on the training data ${D_{train}}$. 
As illustrated in ~\autoref{fig:NAS}(a), the NAS process divides the DNN into fixed parts and a manually defined cell-based search space~\cite{three_SAENAS,ASOC_similar,NASBench-101} $\mathcal{A}$. A search strategy samples a candidate architecture $\alpha$ from $\mathcal{A}$, which is inserted into the backbone network and evaluated. If the number of evaluations (i.e., the compute budget) reaches the predefined limit, the search returns the architecture with the highest validation accuracy; otherwise, it continues.

NAS methods can be divided into three categories: reinforcement learning (RL)-based, gradients-based, and evolutionary computation (EC)-based.
RL-based NAS uses a controller to sample new architectures by making sequential decisions to maximize a reward based on the performance of the architectures. Works such as ~\cite{zhong2020blockqnn}, have revealed that RL algorithms can discover good architectures, but requiring hundreds of GPU hours. Gradient-based NAS~\cite{ASOC_DARTS_GD} relaxes the search space to be continuous and optimizes the architecture and network weights together through gradient descent. One major issue with gradient-based NAS is the \textit{architecture discretization gap}, which arises from the discrepancy between the mixed architecture parameters optimized during the search phase and the discrete architecture used during the evaluation phase, potentially leading to degraded performance of the discretized network. 

Given that Evolutionary Algorithms (EAs) are good at solving complex non-convex optimization problems and expensive optimization problems, EAs are frequently used to solve NAS problems, known as Evolutionary Neural Architecture Search (ENAS).
As illustrated in ~\autoref{fig:NAS}(b), ENAS encodes each neural architecture as an individual within a population. The fitness, typically measured by validation accuracy, guides the evolutionary process across generations to discover high-performing architectures.
However, evaluating numerous DNNs in each generation incurs substantial computational cost, which scales with both population size and the number of iterations, making ENAS resource-intensive. For instance, LargeEvo requires approximately 2750 GPU-days to complete a single search~\cite{real2017large}.

To alleviate the burden of expensive evaluations, recent work has introduced neural architecture performance predictors that estimate candidate performance without full training. Most of them follow a similar training pipeline: a set of DNN is first sampled and trained to obtain their performance; then, a regression model—typically using mean squared error (MSE) as the objective—is constructed based on the architectural encodings and corresponding performance values; finally, the trained regressor is used to predict the performance of newly generated DNNs. By replacing the costly full training process with a performance predictor, the computational burden of ENAS can be significantly reduced. However, building such predictors still requires a large amount of supervised data obtained via expensive fitness evaluations. For instance, NARQ2T~\cite{guo2023latency} requires thousands of evaluation samples to train its surrogate model effectively. Current predictor-assisted ENAS methods still face the following key challenges:
(1) Collecting sufficient training data for the predictor remains computationally expensive, as each training sample demands full training epochs, which limits practical applicability;
(2) Inaccuracies in the predictor may mislead the evolutionary search, resulting in wasted evaluations and suboptimal architectures.

 In order to achieve higher evaluation accuracy while reducing the assessment cost, we propose predictor-assisted ENAS with Dual Contrastive Learning (DCL-ENAS). This method uses contrastive self-supervised learning to learn information flow knowledge from a large amount of unlabelled structural data and derive meaningful neural architecture representations, thereby reducing the amount of effort for fully training networks with labeled data. Using the pre-trained representations, the predictor with high accuracy can be easily trained during the subsequent stage. Specifically, the predictor model's training is performed in a contrastive learning manner to further boost performance, which justifies the `dual' aspect of our method's name.

 The main contributions of this article can be summarized as follows: 
 \begin{itemize}
     \item[1)] In the contrastive pre-training phase, we introduce a novel prior task to generate similar representations for similar architectures, and vice versa. A hard encoder is designed to produce information flow vector representations that serve to group architectures within the search space. Subsequently, a learnable soft encoder, forming part of the predictor model, is developed with the objective of encoding the grouping knowledge derived from the information flow vectors. Notably, this process is self-supervised, obviating the need for labeled data regarding architecture performance, thereby significantly reducing computational costs. Moreover, to thoroughly explore the feature space, the design of the network deliberately incorporates attention mechanisms between nodes and along the information flow paths.
     \item[2)] In the constrastive full-training phase, the entire predictor model, including the pretrained representation component, is fine-tuned to rank different architectures. Unlike the previous predictors that adopt regression methods to directly predict the performance of architectures using $n$ labeled data points, the contrastive learning approach effectively expands the training dataset to $n(n-1)$ pairs. This expansion effectively leverages the data, resulting in a highly accurate model that provides effective guidance for ENAS and reduces evaluation costs.
     \item[3)] In the evolution process, we define new evolutionary operators based on the information flow in architecture graphs. This enhances the algorithm's search capability by generating more diverse architectures and helps in finding high-quality solutions.
 \end{itemize}

The remainder of this article is organized as follows. 
\Autoref{sec:bg} summarizes background and related work.
\Autoref{sec:algo} details the proposed algorithm.
The experimental comparisons and analyses are provided in \Autoref{sec:experiments}. \Autoref{sec:application_study} describes a real-world ECG arrhythmia-classification case study, which further validates our method.
\Autoref{sec:conclusion} draws the conclusions and outlines future work.

\section{Background and related work}
\label{sec:bg}

\subsection{Evolutionary NAS}
\label{sec:rbfn}

ENAS algorithms generally follow the classical evolutionary computation paradigm, which involves several iterative stages. The process begins by initializing a population of individuals, where each individual encodes a distinct deep neural network architecture sampled from a manually defined search space. Subsequently, each candidate architecture is allocated substantial computational resources for training and validation to assess its performance, which serves as its fitness. Based on these fitness values, a subset of top-performing individuals is selected to serve as parents for the next generation. New offspring architectures are then generated by applying genetic operators such as crossover and mutation. This evolutionary cycle—consisting of evaluation, selection, and reproduction—continues until a predefined number of generations is reached. Upon termination, the architecture with the highest fitness is returned as the final output.

Traditional ENAS involves a three-level nested evaluation process: (1) iterating over each individual in the population, (2) performing multiple training epochs per individual, and (3) executing forward and backward propagation for all batches in each epoch. For example, on the commonly used CIFAR-10 benchmark dataset (containing 50,000 samples and a typical batch size of 128), there are \(\lceil {50000}/{128} \rceil = 391\) batches per epoch, where \(\lceil \cdot \rceil\) denotes the ceiling operator. With 500 epochs, each individual undergoes approximately \(2 \times 10^{5}\) training iterations. On a single GPU, each full training may take about 17 hours. Given typical ENAS settings (e.g., 1,000 candidate architectures), the total computation time can reach 700 GPU-days. Even with 20 GPUs, the process would take approximately 35 days, which is still infeasible for users with limited resources.
{ 
In this paper, the term ``compute budget''(maximum number of fitness evaluations in EAs) refers to the maximum number of candidate architectures that are fully trained and validated during the search—equivalently, the number of available \emph{architecture--performance labels} used to fit and update the surrogate predictor. This notion is orthogonal to the amount of task training data (e.g., images) used to optimize the weights of a given candidate network. For example, on CIFAR-10 (50{,}000 labeled training images across 10 classes and 10{,}000 test images), one may fix a compute budget of 200 evaluations, meaning that at most 200 distinct architectures are trained to completion and validated, regardless of the image dataset size. Accordingly, when we refer to “labels” for training the predictor, we mean architecture--performance pairs rather than labels of the underlying images. This clarification is important because the compute budget directly caps the number of costly full trainings.}

The introduction of performance predictors addresses this bottleneck: they enable fast and accurate performance estimation of candidate architectures,
dramatically reducing evaluation overhead.

In terms of selection, ENAS adheres to the “survival of the fittest” principle in evolutionary computation, consisting of parent selection and environmental selection. Parent selection chooses individuals from the current population to generate offspring via crossover and mutation, while environmental selection merges parent and offspring populations and retains the top-performing individuals for the next generation.

\subsection{Predictor-Assisted NAS}
Recent studies have proposed the use of neural predictors to accelerate the estimation of architectural performance. NAS-EA-FA~\cite{pan2021neural} proposed enhancing the performance of the predictor model by utilizing data augmentation and diversity of neural architectures. CTFGNAS~\cite{ASOC_GNN} designs a two-level search space at both the graph and node levels, employing a graph-kernel-based surrogate predictor to estimate the performance of GNNs, which yields promising results. 
MLP-GNAS~\cite{MLP-GNAS} introduces a CNN-based performance predictor to estimate the accuracy of deep learning models without fully training them for an entire cycle, thereby reducing the search time required to identify the best model for a given task. Lupión et al.~\cite{resource-constrained_NAS_pre} proposed a predictor that leverages online learning of early-stage training behaviors of candidate networks to rapidly identify architectures unlikely to deliver performance improvements, thereby reducing redundant computations and accelerating the NAS search process. An inaccurate predictor may mislead the population evolution toward suboptimal directions; thus, enhancing the reliability of performance prediction is crucial to the success of ENAS. 
Some studies have suggested the use of contrastive learning to train surrogate models for architecture performance evaluation. 
For example, Xie et al.~\cite{xie2023architecture} increased the training data through graph isomorphism theory to alleviate the impact of predictor model errors. TCMR-ENAS~\cite{three_SAENAS} adopts a novel triplet contrastive surrogate model combined with a score-based performance evaluation method to predict the ranking of architectures within each group.

\subsection{Self-Supervised Learning and Its Application to NAS}
Self-supervised learning can be divided into four major categories~\cite{balestriero2023cookbook}: (1) Deep Metric Learning~\cite{Self-Supervised_Generative_Contrastive}: This method trains the network by making the embeddings of inputs closer (or farther) to predict whether two inputs belong to the same category. 
(2) Self-supervised methods based on self-distillation~\cite{Deep_clustering} input a piece of data into two encoders to produce two different views, and map one to the other through a predictor.(3) Canonical Correlation Analysis (CCA)~\cite{Canonical_Correlation_Analysis}: The high-level goal of CCA is to infer the relationship between two variables by analyzing their cross-covariance matrix. (4) Masked Image Modeling (MIM): Mask a part of the image and teach the model to complete it \cite{oquab2024dinov}. 

To address the issue of high computational cost in NAS, 
some recent methods have integrated unsupervised or self-supervised learning into NAS. The embedding of the neural architecture is obtained through unsupervised representation learning and used to train the predictor model. Arch2vec~\cite{Arch2vec} assumes that the embedding of the architecture follows a Gaussian distribution and uses a variational autoencoder to reconstruct the input neural architecture, but this assumption cannot be guaranteed. NASGEM~\cite{NASGEM} uses an autoencoder to map the architecture to the embedding space and improves feature representation by minimizing reconstruction loss and similarity loss. However, NASGEM only vectorizes the adjacency matrix of the input neural structure to the embedding space, ignoring the node operations that are crucial to network performance. 
The SSNENAS~\cite{self-supervised-nas} designed a self-supervised learning method, and the architecture obtained under the restriction of a fixed budget is superior to most of the ENAS. SAENAS-NE~\cite{SAENAS-NE} designed an unsupervised training method for the predictor model based on graph2vec and proposed an online Surrogate-Assisted Evolutionary Algorithm (SAEA) filling criterion that deals with the trade-off between convergence and model uncertainty to label promising individuals during the evolutionary process.

\subsection{Research Motivation}
We noticed that some previous studies also attempted to utilize unlabeled architectures. However, unsupervised architectures are prone to generating a large amount of meaningless structures, and self-supervised architectures overlook the isomorphism of neural architectures as well as the broader and context-dependent variations of neural architectures in the search space. In contrast, we revisit neural architectures from the perspective of information flow and conduct representation learning based on the similarity between different semantic transformations of the same architecture in identified batches. This approach can capture more subtle and operationally important patterns that traditional single architecture transformations may overlook. Additionally, we employ a two-stage contrastive learning model to not only enhance the accuracy of the predictor by precise feature differentiation and grouping in the self-supervised stage but also efficiently adjust the model in the contrastive fine-tuning stage by focusing on relative architectural performance. This dual strategy not only reduces reliance on a large amount of labeled data but also ensures that the prediction model is both highly accurate and computationally efficient.

\section{Proposed Algorithm}
\label{sec:algo}

\subsection{Framework}
\begin{figure*}[t]
    \centering
    \includegraphics[width=0.95\textwidth]{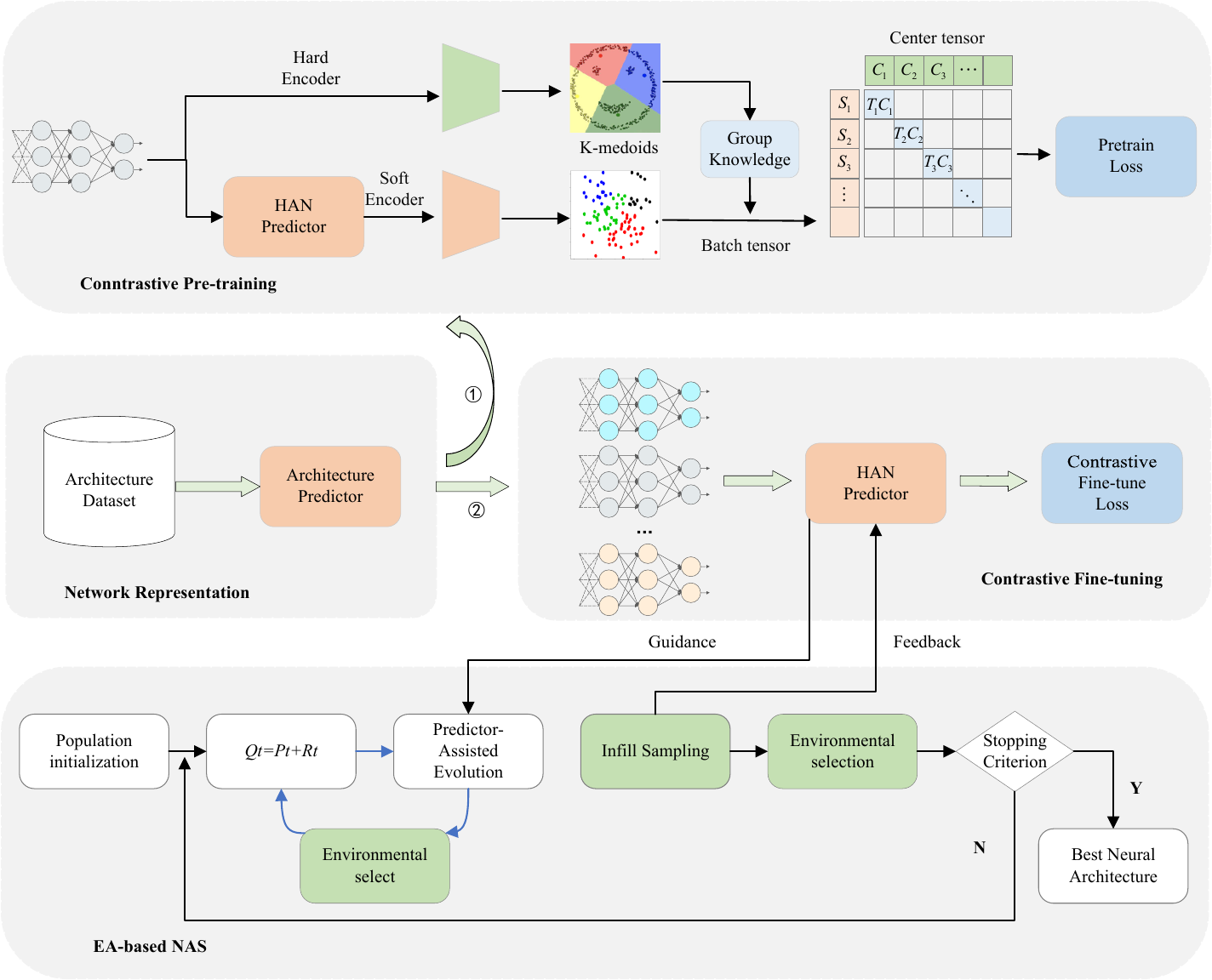}
    \caption{Overview of the proposed DCL-ENAS framework, which consists of a contrastive pretraining stage (CLP) and a contrastive fine-tuning with evolutionary search stage (CLF-ENAS).}
    \label{fig:framework}
\end{figure*}

\begin{algorithm*}[t!]\algsetup{linenosize=\tiny} \footnotesize
    \caption{DCL-ENAS}
    \label{algo:framework}
    \input{DCL-ENAS}
\end{algorithm*}

As illustrated in \autoref{fig:framework}, DCL-ENAS consists of two main stages: contrastive pre-training stage~(\autoref{subsec:pretrain}) and contrastive fine-tuning and evolutionary search stage~(\autoref{subsec:SAEA}).

\begin{itemize}
    \item Contrastive Pre-training (CLP) Stage:
    Before the evolutionary search process begins, the collection of all neural structures in the search space is considered as unlabeled data for pre-training our neural architecture predictor model without having obtained their validation accuracy. Specifically, for a batch of neural structures in the search space, we first encode them using our designed hard encoder (\autoref{subsubsec:hard}) and soft encoder (\autoref{subsubsec:soft}) to obtain the hard and soft encoding of each structure. Subsequently, we cluster the hard encodings, where the learning objective of the soft encoding is to be as close as possible to the prototype (center structure of the group) within each group and as far away as possible from the prototypes of other groups(\autoref{subsubsec:self-supervise}). The soft encoder here is part of the neural predictor, aiming for the neural architecture predictor to learn meaningful representations of neural architectures.
    \item {Contrastive Fine-tuning and Evolutionary Search (CLF-ENAS) Stage:} In this stage, we employ the soft encoder in our predictor model and load the pre-trained weights. During the search process, the neural architecture predictor uses the true validation accuracy based on descendant architectures as fine-tuning data to learn the accuracy ranking of neural architectures (\autoref{subsubsec:finetune}). Descendant neural architectures are generated during the evolution process (\autoref{subsubsec:produce_off}), sampled according to an infill criterion~(\autoref{subsubsec:fill_sample}). The fine-tuned neural architecture predictor is used to estimate the performance of neural architecture individuals in the environment selection process (\autoref{subsubsec:environment_select}). We obtain a true accuracy of the neural architecture as a fitness evaluation. The algorithm terminates and generates an optimal individual once the maximum fitness evaluation budget is reached.
\end{itemize}

The detailed procedure of DCL-ENAS is shown in Algorithm~\ref{algo:framework}. In the CLP stage (Lines 1), contrastive learning is conducted over the architecture space $\mathcal{A}$ to pretrain the soft encoder, yielding the initial parameters $W^{\text{pre}}$ for the predictor model. Then, $N$ architectures are randomly sampled from $\mathcal{A}$ and fully evaluated on a validation set to construct the initial labeled dataset $D_{\text{init}}$. This dataset is used to fine-tune the predictor $M$ and initialize the labeled population $P_{\text{label}}$, while the evaluation counter is set to $fes = |D_{\text{init}}|$. The CLF-ENAS stage (Lines 8–25) alternates between surrogate-assisted evolutionary search and performance evaluation. In each outer iteration, an evolutionary population $P_t$ is derived from $P_{\text{label}}$, and $t_{\text{gap}}$ generations of evolution are performed using crossover and mutation. The candidate offspring and parents are merged to form $R_t$, and the predictor $M$ is employed to estimate their performance. The top $N$ architectures are selected to construct $P_{t+1}$. This inner loop is surrogate-only and incurs no evaluation cost. After every $t_{\text{gap}}$ generations, fitness evaluations are triggered. A subset of $N_{\text{infill}}$ architectures is selected from the current population via infill sampling, considering both predicted scores and uncertainty estimates. These candidates are fully trained and validated to obtain true accuracies, which are then added to $D_{\text{total}}$ and used to further fine-tune $M$. The labeled population is updated by selecting the top $N$ architectures from $P_{\text{label}} \cup P_{\text{infill}}$ based on true accuracy.
The process repeats until $fes$ reaches $fes_{\text{max}}$\footnote{{ Compute budget equals the maximum number of fully trained and validated architectures \(fes_{\max}\), which directly controls the number of architecture--label pairs used for fine-tuning the predictor.}}, after which the best-performing architecture $P_{\text{best}}$ is selected from $D_{\text{total}}$ and returned.

\subsection{Contrastive Learning Pretraining Stage}
\label{subsec:pretrain}

We first introduce the hard encoder of neural network architectures and the soft encoder to be applied within the neural architecture predictor (the following is referred to as the predictor model). Then, we introduce how to adjust the soft encoder through self-supervise learning of grouping knowledge generated from the hard encoder.
\subsubsection{Hard Encoder}
\label{subsubsec:hard}
\begin{figure}[t]
    \centering
    \includegraphics[width=1\textwidth]{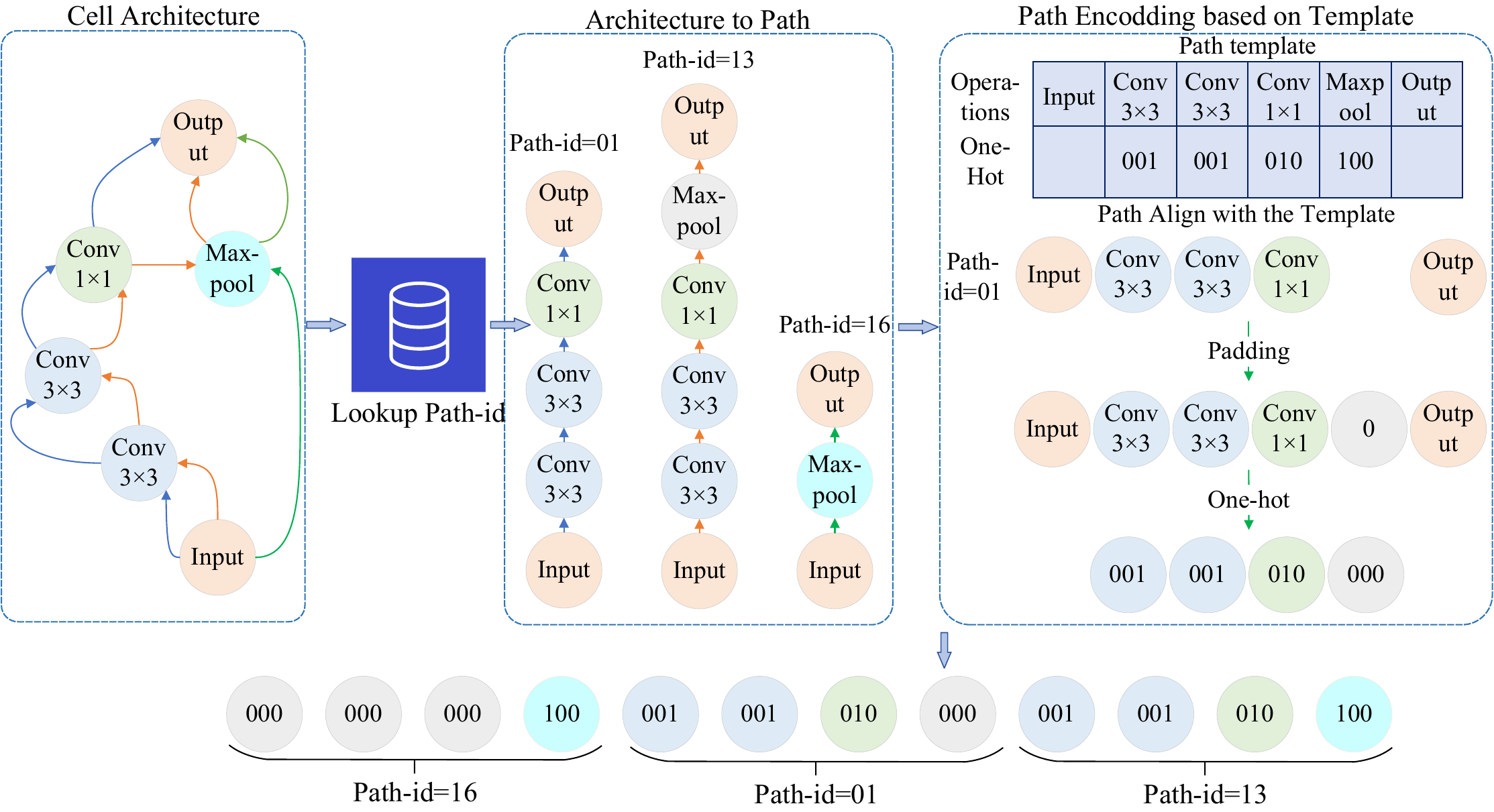}
    \caption{Illustration of the hard encoding process: each architecture is encoded into a unified binary vector based on a predefined path template and corresponding path identifiers (Path-ids).}
    \label{fig:hard_coding}
\end{figure}

Our work utilizes cell-based search spaces, as detailed in NASBench-101 and NAS-\allowbreak Bench-201 \cite{bench201}, 
where each cell's structure is represented as a Directed Acyclic Graph (DAG)~\cite{Parallel_NAS_DAG}. Within this DAG, each node is capable of executing specific operations, such as various convolutional or pooling functions. To capture and represent the flow of information within the network structure, we conceptualize the entire architecture as a collection of multiple information flow paths extending from the input to the output nodes. As depicted in \autoref{fig:hard_coding}, this typical architecture includes three distinct paths.

Given a predefined search space with a fixed operator ordering, we enumerate all possible paths and assign each a unique and consistent identifier, referred to as the \emph{Path-id}, which remains fixed throughout the search process. For instance, in a search space with 1000 distinct paths (after deduplication), we index these paths from 0 to 999 and store them in a database for later reference. Additionally, we define a global operation list, referred to as the \emph{path template}, which specifies a fixed order of all possible operations to align the dimensions of all path encodings. For example, a template such as \texttt{[Input, Conv3×3, Conv3×3, Conv1×1, Maxpool, Output]} would result in intermediate nodes (excluding input/output) being represented by fixed-length one-hot vectors, e.g., \texttt{[001,001,010,100]}.

During encoding, each architecture is decomposed into its valid paths. Each path is encoded as follows: for each operator along the path, its corresponding one-hot vector is retrieved from the template; positions present in the template but not covered by the path are filled with zero vectors to maintain a consistent length. For instance, a path with three valid operations (e.g., Path-id = 01) is encoded as \texttt{001 001 010 000}, comprising three one-hot vectors followed by one zero vector. Similarly, a path with only a single \texttt{Maxpool} operation (e.g., Path-id = 16) is encoded as \texttt{000 000 000 100}, with zero vectors for the first three positions.

To ensure the uniqueness and comparability of encoded architectures, all path encodings are sorted first by the number of valid operations and then by ascending Path-id. For example, a set of paths with Path-ids 01 (3 ops), 13 (4 ops), and 16 (1 op) would be reordered as $\langle 16 \rangle$, $\langle 01 \rangle$, $\langle 13 \rangle$. Considering that different architectures may contain varying numbers of paths, we define a fixed upper bound $L_{seq}$ for the maximum number of paths per architecture. Architectures with fewer than $L_{seq}$ paths are padded with zero vectors (representing identity mappings) to maintain uniform vector dimensions across all encodings.

Finally, we concatenate all path encodings, which serves as the hard encoding for the architecture. Our innovative hard encoding method translates each specific architecture into a unique binary vector. This approach facilitates direct comparisons between different architectures. By applying the Manhattan distance to these binary vectors, we can quantitatively assess the similarity or difference between any two architectures. 

The specific assignment of the one-hot vectors (such as assigning \texttt{Conv3$\times$3} as [001], \texttt{Conv1$\times$1} as [010], and \texttt{MaxPool} as [100]) is arbitrary; any distinct one-hot encoding is valid. The essential requirement is that each operation type receives a unique vector, and the encoding scheme remains globally consistent across the entire search space. 

\begin{proposition}[Operation–code permutation invariance]\label{prop:perm_inv}
    Let \(\phi_1, \phi_2 : \mathcal{O}\!\to\!\{0,1\}^{|\mathcal{O}|}\) be two bijective one-hot
    encodings that differ only by a permutation of coordinates.  
    For any two paths (or architectures) \(x,y\),
    \[
    d_{\text{manhattan}}\!\bigl(\Phi_{\phi_1}(x),\Phi_{\phi_1}(y)\bigr)\;
    =\;
    d_{\text{manhattan}}\!\bigl(\Phi_{\phi_2}(x),\Phi_{\phi_2}(y)\bigr).
    \]
\end{proposition}

\begin{proof}
By definition of one-hot encoding, $\phi_a$ and $\phi_b$ differ only by a permutation of vector positions. For any operation $o_i \in \mathcal{O}$, there exists a permutation matrix $P$ such that $\phi_b(o_i) = P \cdot \phi_a(o_i)$. For any path $p$, its representations under the two encodings can be written as: $\Phi_{\phi_a}(p)$ = $[\phi_a(o_{p1})$, $\dots$, $\phi_a(o_{pL})]$,
    $\Phi_{\phi_b}(p)$ = $[P\phi_a(o_{p1})$, $\dots$, $P\phi_a(o_{pL})]$ = $P' \cdot \Phi_{\phi_a}(p)$. where $P'$ is the permutation matrix extended to the entire path encoding vector, which only permutes the order of vector components without altering their values. By the definition of Manhattan distance, for any two paths or architectures $x,y$, the distance is the sum of absolute differences of the corresponding components:
    $d_{\text{manhattan}}(\Phi_{\phi}(x), \Phi_{\phi}(y))$ =
    $\textstyle\sum_{j} \left| \Phi_{\phi}(x)_j - \Phi_{\phi}(y)_j \right|$, 
    Since the $|a-b|$ operation is independent of the order of components, we have:
    $d_{\text{manhattan}}\left( \Phi_{\phi_b}(x), \Phi_{\phi_b}(y) \right)
    $ = $d_{\text{manhattan}}\left( P' \cdot \Phi_{\phi_a}(x),\; P' \cdot \Phi_{\phi_a}(y) \right)$
    = $d_{\text{manhattan}}\left( \Phi_{\phi_a}(x),\; \Phi_{\phi_a}(y) \right)$.
    This derivation shows that any permutation of the one-hot vectors for operations does not affect the computed Manhattan distance, which completes the proof.
\end{proof}

\subsubsection{Soft Encoder}
\label{subsubsec:soft}
\begin{figure}[t]
    \centering
    \includegraphics[width=1\textwidth]{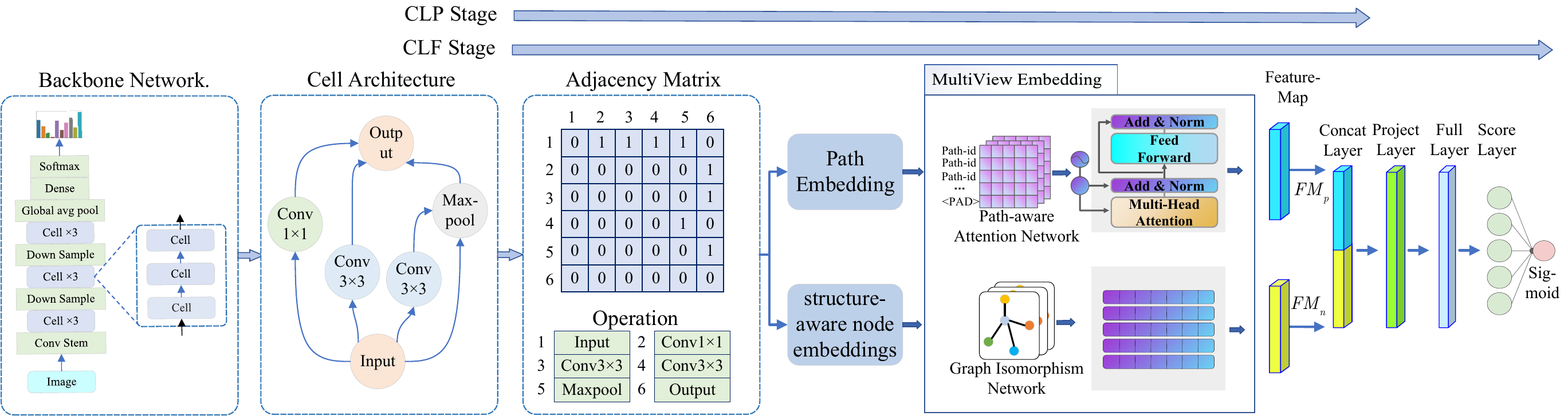}
    \caption{The architecture of the DCL-ENAS neural architecture predictor. Each candidate architecture is represented from two distinct perspectives, resulting in two fixed-length feature maps. These are subsequently processed through four layers—Concat, Project, Full, and Score—to produce a final ranking score. The final "Score" layer employs a Sigmoid activation function to normalize the output into the range $[0, 1]$. During the CLP stage, predictor parameters are updated up to the Concat layer, while the CLF stage introduces and optimizes new parameters specific to the Score layer.}
    \label{fig:softencoder}
\end{figure}
The predictor of DCL-ENAS considers attention at multiple scales in order to achieve more comprehensive feature extraction of neural architectures. The structure of the predictor model is shown in \autoref{fig:softencoder}. A neural network architecture will get two feature map tensors $FM_n$ and $FM_p$ after being represented from the perspectives of node attention and path attention. Then it goes through four layers of neural networks (Concat, Project, Full, Score) to get a scalar output. Note that there should be a non-linear transformation $ReLU$ layer after the project layer. To ensure both interpretability and boundedness of the predicted performance scores, a Sigmoid activation function is applied at the final output layer of the predictor model. As a result, each predicted score $\hat{y}_i$ is constrained within the range $[0,1]$. This normalization facilitates direct comparison among candidate architectures and is particularly suitable for ranking operations in the evolutionary selection process. The soft encoder, a key component of the predictor model, spans from the input to the Concat layer.

a)~Node attention: When calculating node attention, a neural network architecture is considered from the perspective of graph neural networks. First, we use the Graph Isomorphism Network (GIN) algorithm~\cite{GIN} to extract the feature vector of each operation node in each neural network architecture: 
\begin{equation}
        {H}_v^{(m)} = \text{MLP}^{(m)} \left( (1 + \epsilon^{(m)}) \cdot {H}_v^{(m-1)} + \sum_{{U} \in \mathcal{N}(v)} {H}_{U}^{(m-1)} \right)
\end{equation}
where ${H}_v^{(m)}$ represents the feature of node $v$ at the $m$-th layer. $\mathcal{N}(v)$ is the set of neighbor nodes of $v$, $\text{MLP}^{(m)}(\cdot)$ is a multilayer perceptron, and $\epsilon^{(m)}$ is a learnable parameter. An embedding of a set of nodes is obtained by weighted averaging to get an embedding of a graph. However, which nodes are more important and should get more weight depends on a specific similarity measure. Therefore, in order for the model to learn weights guided by a specific similarity measure, we propose the following attention mechanism.

First, compute a global graph context $\vect{c}\in\mathbb{R}^D$, which is a simple average of node embeddings followed by a non-linear transformation:
\begin{equation}
\vect{c} = \tanh\left(\frac{1}{{v_s}}\mat{W}\sum_{v=1}^{v_s}{H}_v\right)
\end{equation}
where $\mat{W}\in\mathbb{R}^{D\times D}$ is a learnable weight matrix. The global graph context $\vect{c}$ provides global structural and feature information of the graph, which is adapted to the given similarity measure by learning the weight matrix. Based on $\vect{c}$, we can calculate an attention weight for each node.
In order for the attention $a_v$ of node $v$ to be aware of the global context, we take the inner product of $\vect{c}$ with its node embedding. Apply the sigmoid function $\sigma(x)$ to the result to ensure that the attention weight is within the $(0,1)$ range. Finally, the graph embedding ${{H_g}}\in\mathbb{R}^D$ is the weighted sum of node embeddings. The following equation summarizes the proposed node attention mechanism:

\begin{equation}
{H_g}=\sum_{v=1}^{v_s}\sigma({H}_v^T\vect{c}){H}_v
\end{equation}
Finally, output a $d_{FM}$-dimensional Feature Map tensor $FM_n$.

b)~Path attention: As elaborated in Hard Encoder (\autoref{subsubsec:hard}), we can also consider an architecture as a network of multiple paths of information flows. Each path possesses a binary hard encoding, and 
is embedded into the feature space as a `Token'. 
For each path $p$ in the path table $\mathcal{E} $ of the search space, we map it to a vector $\vect{e}_{p}\in\mathbb{R}^{D_{e}}$ using the lookup table $\vect{E}\in\mathbb{R}^{D_{e}\times|\mathcal{E}| }$, where $D_{e}$ is a hyperparameter representing the dimension of token embeddings. These path embeddings can be trained and learned together with other model parameters.

To ensure that the tensor shapes of path embeddings of all neural architectures in the search space are the same, we stipulate that a neural architecture contains at most $L_{seq}$ tokens. For neural architectures with a number of paths less than $L_{seq}$, a special token $\langle$PAD$\rangle$ will be padded.
After embedding processing, an architecture can be represented as a list of path embedding vectors [$\vect{e}_{1},\vect{e}_{2},\cdots,\vect{e}_{L_{\text{seq}}}$]. 
We leverage the encoder component of the transformer architecture, as detailed in~\cite{transformers}, to facilitate interaction among these path embeddings (\autoref{fig:path-encoder}). This interaction allows each path to attend to information from the others, enhancing the representation by integrating contextual information from the entire architecture.
Finally, the path attention module outputs a $d_{FM}$-dimensional Feature Map tensor $FM_p$.

\begin{figure}
    \centering
    \includegraphics[width=0.7\textwidth]{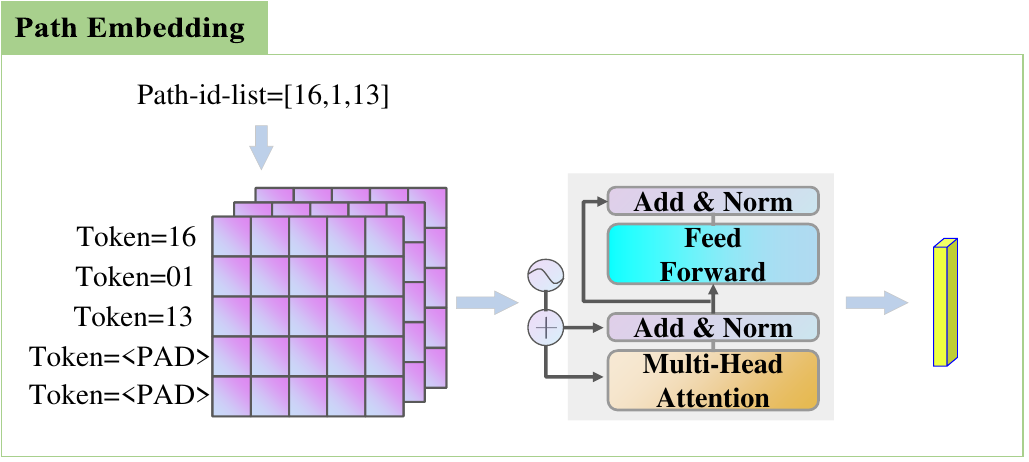}
    \caption{In search space, all paths will have a unique path index (Path-id). Therefore, for a neural structure, we can obtain the Path-id list $P_{\text{ids}}$ of this structure (the same as the hard-coded processing method, multi-level sorting of the path is required), and then the path index Path-id is embedded into the feature space as a Token. We stipulate that an architecture contains at most $L_{\text{seq}}$ paths. If the number of paths in an architecture is less than $L_{\text{seq}}$, a special token "$\langle$PAD$\rangle$" will be filled. The path identifier for the special token "$\langle$PAD$\rangle$" is a very large positive integer.
    }
    \label{fig:path-encoder}
\end{figure}

\subsubsection{Self-Supervised Pretraining of Soft Encoder}
\label{subsubsec:self-supervise}

The parameters associated with our soft encoder are learnable and can be adjusted during the training phase of our predictor model. To minimize the reliance on extensive labeled data of network performance, we initially pre-train the soft encoder in a self-supervised manner. This pre-training procedure utilizes grouping knowledge of similar architectures, derived from their hard representations, to facilitate contrastive learning. 

Specifically, our hard encoder employs $K$-medoids clustering within a batch to categorize architectures into groups. Given that hard encodings are binary vectors, we use the Manhattan distance to determine the distances required for clustering in $K$-medoids. Following the clustering, the soft encoder is pre-trained using the acquired grouping knowledge. Within the same batch, an architecture, denoted as $\alpha_i$, pairs positively with its group's prototype architecture (the center architecture of its group) and negatively with prototypes from other groups. Then, contrastive learning is performed to enhance the similarity score of the positive sample pair, while reducing the similarity of the negative sample pair.
The process of self-supervised learning can be seen from \autoref{fig:framework}.

\begin{definition}{Pretraining Loss:}
    \label{ploss}
    For a batch of $bn$ samples, consider a query (soft encoding vector) $E_{i}$ and a set of center encodings $\{E_{c_{1}},E_{c_{2}},\cdots,E_{c_{{K}}}\}$, where $c_k$ denotes the index of a prototype point. %When it is similar to the positive key $E_{k_{i}}$ and dissimilar to other keys, the loss is a low value. Dot product is used to measure similarity:
    The loss is minimized when $E_i$ is similar to its corresponding positive key $E_{c_j}$ (suppose the query belongs to the $j$-th group) and dissimilar to other keys:
    \begin{equation}
        L_{\rm{pretrain}}=\sum\limits_{i = 1}^{bn} {
        -\log\frac{\exp(E_{i}\cdot E_{c_{j}}/\tau_{c_{j}})}{\sum_{{k}=1,{k}\neq j}^{K}\exp(E_{i}\cdot E_{c_{{k}}}/\tau_{c_{k}})} 
        }
    \end{equation}
    where $\tau_{c_{k}}$ represents the Crowding Distance of the $k$-th center point. The sum contains one positive sample and $(K-1)$ negative samples. This loss is a logarithmic loss of a softmax-based $K$-classifier that tries to classify $E_{i}\cdot E_{c_{j}}$ as belonging to the group $j$. 
    To ensure the soft encoder fully assimilates the knowledge within a batch, multiple pre-training sessions are conducted, varying $K$ and hence the grouping sizes. The losses from these sessions are averaged before backpropagation is performed.
\end{definition} 

Note that the clusters possess different density distributions, and we define the crowding distance of the $k$-th cluster as
\begin{equation}
    \tau_{c_{k}} = \frac{\sum_{g_k=1}^{gs_k} \|\hat{E}_{g_k} - \hat{E}_{c_{k}}\|_{1}}{gs_k\log(gs_k+\beta)}
\end{equation} 
Here, we use $\hat{E}$ to denote the hard encoding of an architecture, $g_k$ iterates over the architectures within the same group with the $k$-th prototype, $gs_k$ denotes the total number of these architectures, and $\beta$ is a smoothing parameter (fixed to 
10) to prevent excessively high crowding in small clusters.

The crowding distance plays a critical role in scaling the inner product $E_i \cdot E_{c_k}$ in the pretraining loss, modulating the influence of different prototypes based on their cluster densities. Specifically, this scaling adjusts the softmax distribution's sharpness, making it steeper in denser clusters (smaller crowding distance) and flatter in sparser ones (larger crowding distance). It adjusts the influence of different data clusters within the learning process, reducing the inherent bias on sparsely populated areas by appropriately scaling the impact of denser regions, thus promoting a more balanced representation learning across the entire architectural space.

\subsection{CLF-ENAS stage}
\label{subsec:SAEA}

\subsubsection{Fine-tuning of the Predictor Model}
\label{subsubsec:finetune}
During the process of evolutionary neural architecture search, the neural architecture predictor needs to be fine-tuned using real data related to its performance. 
 When training predictors, commonly used loss functions are element-wise mean squared error (MSE) or L1 loss functions \cite{Npenas,SSNENAS,GCN-predictor,SAENAS-NE}. They assume that lower MSE or L1 loss leads to better ranking results. However, this is not always correct. For example, consider two networks with actual classification accuracies of 0.9 and 0.91 on a validation set. In the first case, they are predicted as 0.91 and 0.9, while in the second case, they are predicted as 0.89 and 0.92. The MSE losses in both cases are the same, but the former is worse because there is a change in ranking between the two networks, and the search strategy will select the architecture with poorer performance. We believe that the ranking of prediction accuracy between different architectures is more important than their absolute performance. Given a predictor $\varepsilon$ and two different architectures $\alpha_{1}$ and $\alpha_{2}$, where $\varepsilon(\alpha;w)$ represents the predictive performance of architecture $\alpha$ with weight $w$, the following relationship should hold:
$
\varepsilon(\alpha_{1};w) > \varepsilon(\alpha_{2};w)$
if and only if
$ACC(\alpha_{1},D) > ACC(\alpha_{2},D)$,  where $D$ is the training set with accuracy labels.
This implies that the predictor should correctly rank different network architectures based on their actual performance.

Here, we adopt the principle of contrastive learning to minimize the requirement for labeled data by focusing on relative positions rather than absolute values. Our pairwise loss function is defined below.

\begin{definition}{Fine-tuning Loss:}
For each neural architecture $\alpha_{i}$ and $\alpha_{j}$ ($i \neq j$) within a batch, let their corresponding real target values be $y_{i}$ and $y_{j}$, and the output of the predictor's score layer be $\hat{y}_i$ and $\hat{y}_j$. The loss function, $L_{\rm{finetune}}$, is defined as
\begin{equation}
    L_{\rm{finetune}} = \sum_{i=1}^{N} \sum_{j=1,j \neq i}^{N} 
    \mathbbm{1} [ \text{sign}(\Delta \hat{y}_{ij}) \neq  \text{sign}(\Delta y_{ij}) ] \cdot ( |\Delta \hat{y}_{ij}|)
\end{equation}

Here, $\Delta y_{i j}=y_i-y_j$ and $\Delta \hat{y}_{i j}=\hat{y}_i-\hat{y}_j$ represent the actual and predicted differences in performance between architectures $\alpha_i$ and $\alpha_j$, respectively. The loss is incurred only when the signs of $\Delta y_{i j}$ and $\Delta \hat{y}_{i j}$ differ, indicating a discrepancy in the predicted ranking relative to the actual ranking. If the signs match, meaning the relative ranking is correctly predicted, the loss is zero. This loss function prioritizes the accurate prediction of the relative order rather than absolute performance metrics, effectively expanding the { supervision} to $n(n-1)$ { pairwise constraints, thereby extracting more signal from the same compute budget} commonly faced in neural architecture search tasks.

\end{definition}

\subsubsection{Producing Offspring}
\label{subsubsec:produce_off}

\begin{figure*}[t]
    \centering
    \includegraphics[width=1\textwidth]{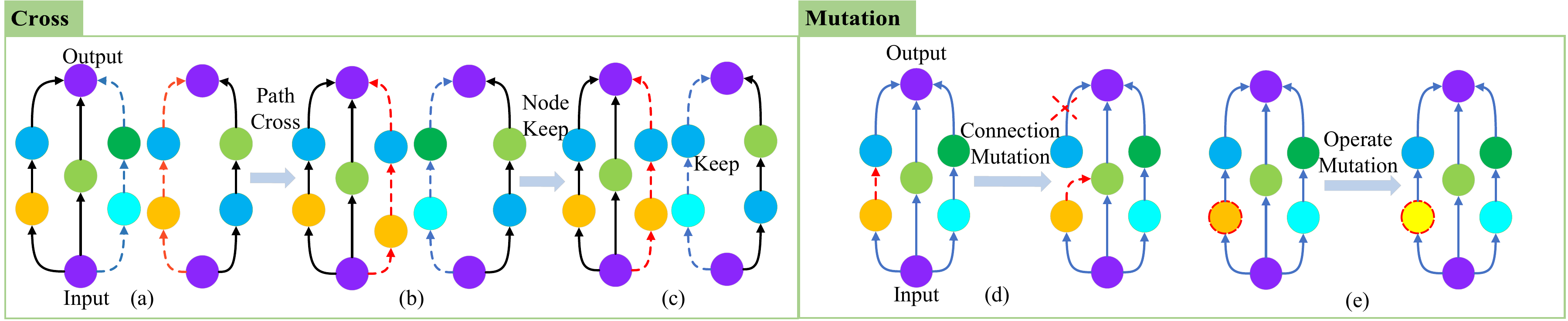}
    \caption{Illustration of offspring generation, including crossover and mutation. The crossover is mainly based on path crossover. There are two types of mutations, one is connection mutation, and the other is operation mutation.}
    \label{fig:cross_mutation}
\end{figure*}
The evolutionary operators for producing offspring here include crossover and mutation operations. In DCL-ENAS, selection pressure is imposed during the environment selection step, where we adopt a rank-based strategy: all candidate architectures ($R_t = P_t \cup Q_t$) are sorted by predicted fitness rank, and the top $N$ are retained for the next generation.
In contrast, parent selection for crossover is conducted via uniform random sampling from the current labeled population $P_t$. This strategy avoids premature convergence and maintains diversity, especially during early generations when the predictor model may still be inaccurate. Subsequently, each parent undergoes mutation with a probability $P_{\rm{m}}$. This sequence of operations is applied to each individual in the population, defining what we term as one complete `round' of the process. To ensure comprehensive exploration of the search space, this cycle is repeated over multiple rounds. The process continues until the total number of offspring produced exceeds $C_{\rm{a}}$ times the population size.

a)~Crossover: We design a crossover operation based on information flow paths tailored to the specifics of cell-based neural architecture, as illustrated in \autoref{fig:cross_mutation}. 
 \autoref{fig:cross_mutation}(a) represents a schematic diagram of two cell-based neural architectures, while \autoref{fig:cross_mutation}(b) shows the random switch of paths between them.
 Given the intricacies of the fitness landscape in neural architecture optimization, it is crucial to maintain stability in the evolutionary steps. As depicted in \autoref{fig:cross_mutation}(c), to avoid excessively large deviations in the architectural lineage, certain nodes are retained post-crossover at corresponding positions with a probability $P_{\rm{keep}}$. The retention probability of each node is the frequency of the node appearing in the neural architecture before the crossover.

b)~Mutation: In our approach, we implement two distinct mutation operators. The first, Connection Mutation (\autoref{fig:cross_mutation}(d)), involves randomly altering the target node of a connection within the network. This modification can lead to a cascade effect where subsequent connections in the information pathway may lose data flow, resulting in their deletion to maintain network integrity. The second operator, Operation Mutation (\autoref{fig:cross_mutation}(e)), entails selecting a node at random and replacing its operation with another, also chosen randomly. This process generates diverse offspring networks from the original parent architecture. By alternately applying these two mutation operators, our method enables exploration of the entire architectural space, potentially uncovering highly effective network architectures.

\subsubsection{Environment Selection and Infill Criterion}
\label{subsubsec:environment_select}
\label{subsubsec:fill_sample}
The environmental selection and infill criterion in DCL-ENAS use operations similar to those in the SAENAS-NE algorithm \cite{SAENAS-NE}. When selecting an environment, each parent neural architecture is associated with $C_{\mathrm{a}}$ closest child architectures to form a family. The predictor estimates the validation accuracy ranking of neural architectures within each family, and the best-performing architecture in each family is chosen for the next generation individual. Infill sampling involves non-dominated sorting based on the estimated performance of neural architectures by the neural architecture predictor and the uncertainty of the predictor, then selecting $N_{\text {infill }}$ new neural architectures from the architecture pool and obtaining their actual validation accuracy.

\subsection{Workflow}
\label{Workflow}

The workflow of DCL-ENAS begins with contrastive pretraining over the set of candidate neural architectures in the search space $\mathcal{A}$ (i.e., the CLP stage). During this stage, each architecture $\alpha$ is first transformed into a unique binary vector via a hard encoder, where information flow within the architecture is captured using path identifiers (Path-ids). Based on these representations, $K$-medoids clustering is performed using Manhattan distance to group architectures and extract structural grouping knowledge across the search space. This unsupervised group information is subsequently used to guide the training of a soft encoder. As a component of the neural architecture predictor $M$, the soft encoder is optimized via a self-supervised contrastive loss, which aims to bring architectures within the same group closer to their assigned prototype (i.e., group centroid), while pushing them away from prototypes of other groups.

Upon completing the pretraining stage, the algorithm enters the contrastive fine-tuning and evolutionary architecture search phase (CLF-ENAS).
\begin{itemize}
    \item First, an initial set of $N$ architectures is randomly sampled, denoted as $D_{\text{init}} = \{a_1, a_2, \ldots, a_N\}$. These architectures are fully trained and evaluated on the validation set $D_{\text{eval}}$ to obtain their actual validation accuracies (i.e., ground-truth fitness). The resulting labeled dataset is denoted as $D_{\text{total}}$ and used to fine-tune the predictor model $M$. The initial population is then set as $P_{\text{label}} = D_{\text{init}}$, and the number of fitness evaluations is initialized as $fes = |D_{\text{init}}|$.
    \item The evolutionary search proceeds iteratively. In each iteration, every architecture in the population generates $r$ offspring through crossover and mutation. Crossover operations combine the path structures of two parent architectures selected uniformly at random, incorporating a node preservation probability to retain structural stability. Mutation operations include connection mutation and operation mutation, which respectively alter the connectivity between nodes and the operation type assigned to each node, thereby enhancing the diversity of the search within $\mathcal{A}$. The offspring set $Q_t$ is then merged with the current population $P_t$ to form a candidate set $R_t$. The predictor $M$ is then used to estimate the relative validation accuracies of all architectures in $R_t$. The top $N$ architectures with the highest predicted scores (normalized to $[0, 1]$) are selected to form the next generation population $P_{t+1}$. Every $t_{\text{gap}}$ generations, a actual fitness evaluation is triggered. Based on both predicted performance and uncertainty estimates from $M$, an infill sampling criterion is applied to select $N_{\text{infill}}$ candidate architectures from the current population. These are fully trained and evaluated on $D_{\text{eval}}$ to obtain their true validation accuracies. The number of evaluations is then updated as $fes \leftarrow fes + |P_{\text{infill}}|$, and the newly evaluated samples are added to $D_{\text{total}}$ for further fine-tuning of $M$.
    \item The evolutionary process continues until the total number of actual evaluations $fes$ reaches the predefined budget $fes_{\text{max}}$. Finally, the architecture with the highest observed validation accuracy in $D_{\text{total}}$ is selected as the best architecture $P_{\text{best}}$ and returned as the final output of the search.
\end{itemize}
\section{Experiments}
\label{sec:experiments}

\subsection{Experiment Setup}
In this section, we conduct experiments on two search spaces, NASBench-101 and NASBench-201:
\begin{itemize}
    \item NASBench-101 has designed a cell-based search space, where cells are defined by DAG on nodes. To limit the size of the search space, only 3 × 3 convolutions, 1 × 1 convolutions, and 3 × 3 max pooling are allowed, with a maximum of 9 edges. All architectures have been trained and evaluated three times on CIFAR-10 with different random initializations. In total, there are 423,624 convolutional neural network architectures.
    \item  NASBench-201\footnote{{ The validation and test accuracies of neural architectures are computed following SAENAS-NE: \url{https://github.com/HandingWangXDGroup/SAENAS-NE}.}} is similar to NASBench-101, containing 15,625 architectures. Each architecture is generated by 4 nodes and 5 related options (zeroize, skip-connect, 1 × 1 convolution, 3 × 3 convolution, and 3 × 3 avg-pool). Each node and edge respectively represent feature mapping and operation. NASBench-201 provides training, validation, and testing accuracy for CIFAR-10, CIFAR-100, and ImageNet-16-120.
\end{itemize}

The experiment is divided into three parts. First, we quantify the performance of the proposed DCL-ENAS on the search space and compare it with existing NAS algorithms. Second, we test the Kendall-$\tau$ coefficient of the architecture validation accuracy and actual validation accuracy estimated by different neural architecture predictors. Finally, the effectiveness of DCL-ENAS is verified through ablation experiments.{ 
We emphasize that this study focuses on limited compute budgets (i.e., a cap on the number of fully trained and validated architectures)—rather than on changes to the \emph{underlying task dataset}. In particular, we do not reduce the size of the training dataset (e.g., by substantially downsampling CIFAR-10). Such reductions in the underlying task dataset can heighten the risk of overfitting for automatically selected architectures across the search space. Moreover, the tabulated NAS benchmarks used here—NASBench-101 and NASBench-201—fix the underlying task datasets and do not report results under systematic sample-size reductions; this precludes a controlled evaluation of reduced-data regimes within these benchmarks. }All experiments were conducted using two NVIDIA GeForce RTX 4090 GPUs and an Intel(R) Xeon(R) Gold 6254 CPU @ 3.10GHz.

\subsection{Algorithmic Settings}
\begin{table}[t!]\tiny
    \centering
    \caption{Parameter settings of DCL-ENAS.}
    \label{tab:parameter}

\input{parameter}
\end{table}

An overview of the parameter settings for DCL-ENAS is presented in \autoref{tab:parameter}. The descriptions of some important parameters are as follows.

(1) \textbf{Pretraining stage:} In this stage, large batches of architectures from both search spaces are used as unlabeled data to train the discriminative capability of the soft encoder. For the NASBench-101 search space, the batch size is set to 20,000 and the number of training epochs is 200. In each batch, the soft encoder learns from hard-encoded representations via multiple runs of $K$-medoids clustering, where $K \in \{5, 10, 20\}$. For the more compact NASBench-201 search space, the batch size is reduced to 10,000 with 200 training epochs, and the cluster size $K$ is chosen from the set of integers in the range $[10, 20]$. These hyperparameter settings ensure the availability of sufficient contrastive samples while avoiding excessive computational cost, thereby enabling stable acquisition of pretrained weights. To represent architectures, a GIN-based graph neural network is employed to embed the soft-encoded inputs. The node embedding dimensions are set to 6 for NASBench-101 and 8 for NASBench-201, respectively. Correspondingly, the hard-encoded architecture vectors used for contrastive learning have lengths of 120 and 96. 

(2) \textbf{Fine-tuning and evolutionary search stage:} Both search spaces share the same hyperparameter settings during the fine-tuning and evolutionary search phase, including the number of training epochs (50) and batch size (8192). The only difference lies in the maximum number of fitness evaluations, which is set to 150 for NASBench-101 and 100 for NASBench-201. 
During the evolutionary search phase, the population size is fixed at $N=20$, and $r=6$ offspring are generated per generation. The crossover and mutation probabilities are set to $P_c = 0.9$ and $P_m = 0.1$, respectively, with a node preservation probability of $P_{\text{keep}} = 0.5$.
Fitness evaluations are triggered every $t_{\mathrm{gap}} = 5$ generations: in each cycle, surrogate predictions are used to estimate the fitness of individuals for the first four generations, while architectures in the fifth generation are fully trained to obtain true validation accuracy. To promote population diversity, environmental selection is conducted within kinship clusters defined by a radius of $C_a = 6$.
At each actual evaluation step, the infill strategy selects $N_{\text{infill}} = 10$ promising architectures for full evaluation. The results are added to the predictor’s training dataset for further fine-tuning. We emphasize that we do not study NAS in regimes where the training data itself is scarce, which may raise overfitting concerns for automatically selected architectures. Such regimes are beyond our present scope.

\subsection{Comparison with NAS algrithm}
\label{subsec:offline-ddea}

\begin{table*}[t]
    \centering
    \caption{Search results on the validation set of different benchmarks.}
    \resizebox{1\linewidth}{!}{
    \label{tab:valid} {
        \input{valid}
    }}
\end{table*}

In this subsection, we compare our proposed DCL-ENAS with several NAS. 
We provide a brief description of each of them below.

\begin{itemize}
    \item {Random}: Random search is the simplest yet competitive baseline of NAS algorithms. It randomly selects $fes_{max}$ architectures from the search space and uses the architecture with the highest validation accuracy as the final result.
    \item {REA}~\cite{regularized}: An algorithm with age attributes was proposed, which decides the death of genotypes by age, thus supporting younger individuals. The paper also compared the three image classifier architecture search algorithms of evolutionary algorithms, reinforcement learning, and random search. The research results show that evolutionary algorithms are slightly faster than reinforcement learning in search speed and perform well in situations where resources are scarce.
    \item {BANANAS}~\cite{white2021bananas}: Proposed to combine Bayesian optimization with predictor models for neural architecture search, and introduced a novel path-based encoding scheme. In BANANAS, an integrated feed-forward neural network serves as the predictor model. The variants of BANANAS are BANANAS-AE based on adjacency matrix encoding, BANANAS-PE based on path encoding, and BANANAS-PAPE based on location-awareness, which comes from ~\cite{SSNENAS}.
    \item {BOHAMIANN}~\cite{BOHAMIANN}: This method proposes a general approach to Bayesian optimization using a flexible parametric model (neural network), as close as possible to the true Bayesian process. it achieves scalability through Stochastic Gradient Hamiltonian Monte Carlo, which improves its robustness through scale adaptation. The NAS baseline implementation first appeared in ~\cite{white2021bananas}.
    \item {BONAS}~\cite{BONAS}: This algorithm applies GCN predictor as predictor models for Bayesian optimization, selecting multiple relevant candidate models in each iteration. The candidate models are obtained based on weight sharing and training.
    \item {DNGO}~\cite{DNGO}: In the process of Bayesian optimization, a neural network is used as a basis function to fit the distribution of queries.
    \item {GCN-predictor}~\cite{GCN-predictor}: Random search based on the GCN predictor model.
    \item {GP-BO}: The predictor model is a Gaussian process, and the optimization process is Bayesian. which comes from ~\cite{white2021bananas} and~\cite{Npenas}.
    \item {local-search}~\cite{local-search}: NAS is conducted based on the simplest hill climbing. As the noise decreases, the number of local minima greatly reduces, and theoretical characteristics of local search performance in NAS are provided.
    \item {arch2vec}~\cite{Arch2vec}: Uses reinforcement learning and Bayesian optimization as optimizers to search for the optimal neural architecture, and uses the embedding method arch2vec that adopts a variational autoencoder to learn the representation of neural architectures. The two search methods in the paper are recorded as Arch2vec-RL and Arch2vec-BO.
    \item {NPENAS-NP}~\cite{Npenas}: Designed a graph-based uncertainty estimation network as a predictor model, the second predictor is a graph-based neural network, which directly predicts the performance of the input neural network. 
    \item {NPENAS-SSRL and NPENAS-SSCCL}~\cite{SSNENAS}: NPENAS-SSRL and NPENAS-SSCCL are improvements based on ~\cite{Npenas}, using evolutionary algorithms as optimizers and two self-supervised learning models to pretrain architecture embeddings. One is to learn architecture embeddings by introducing a pretraining task to predict the distance between architectures. The other is to first construct positive samples and negative samples, and then propose a contrastive learning algorithm to learn architecture embeddings.
    \item {SAENAS-NE}~\cite{SAENAS-NE}: proposed an evolutionary algorithm based on a network embedding predictor model. First, it uses an unsupervised learning method (graph2vec) to generate the representation of each architecture, and architectures with more similar structures are closer in the embedding space. Secondly, a new environmental selection method based on the reference population is designed to maintain the diversity of the population in each generation, and a filling criterion dealing with the trade-off between convergence and model uncertainty is proposed to reselect the environment.
    \item {CAP}~\cite{CAP}: A context-aware neural predictor (CAP) based on architectural context information is proposed. The input architecture is encoded as a graph by GIN, and the predictor infers the contextual structure around the nodes in each graph.
\end{itemize}

The NAS results of different algorithms on the validation set of NASBench-101 and NASBench-201 are shown as in \autoref{tab:valid}. 

To comprehensively evaluate the performance differences between DCL-ENAS and multiple NAS baselines across diverse benchmarks, we conducted 100 independent search runs for each algorithm and recorded the best validation accuracy from each run. We then applied the Wilcoxon signed-rank test~\cite{Wilcoxon} to assess the statistical significance of the pairwise performance differences between DCL-ENAS and each baseline on every benchmark (NASBench-101, NASBench-201-CIFAR10-valid, CIFAR100, and ImageNet16-120).
We report the statistical outcomes in the $+/\approx/-$ format, where $+$ denotes that DCL-ENAS performs significantly better than the baseline ($p < 0.05$), $\approx$ indicates no significant difference ($p \geq 0.05$), and $-$ implies significantly worse performance.
For instance, in \autoref{tab:valid}, the row corresponding to SAENAS-NE reads “1/3/0”, indicating that DCL-ENAS significantly outperforms SAENAS-NE on one benchmark, performs comparably on three, and does not underperform on any. Additionally, an entry such as “94.61$\pm$0.10(+)(11)” under the NASBench-101 column signifies that DCL-ENAS significantly outperforms SAENAS-NE (denoted by “(+)”), while “(11)” indicates SAENAS-NE's rank based on the average validation accuracy among all compared methods.

The results show that our algorithm surpasses other algorithms in terms of performance and ranking. Besides our algorithm, the superior performance of NPENAS-SSCCL, NPENAS-SSRL, and SAENAS-NE proved the effectiveness of self-supervised learning and unsupervised methods in neural architecture search. Especially in the ENAS process guided by the neural architecture predictor, the performance of the predictor is crucial. NPENAS-SSCCL and NPENAS-SSRL predictor models focus on predicting the feature extraction performance of a given neural architecture and train the model using the element-level loss function MSE. However, in the neural architecture search, the differences in validation accuracy between the architectures may be minutely small (the accuracies of many architectures even differ by one thousandth of a percent), and there is a bottleneck even if the predictor model has strong feature extraction capabilities. SAENAS-NE extracts features of neural architecture based on unlabeled methods. This type of unlabeled method focuses on graph local information and overlooks global graph information such as isomorphic graph differences. Due to inevitable flaws in manually encoding neural structures directly, BANANAS's performance is somewhat disappointing: it records the position of each operation by assigning a unique index to each node. But in the same network architecture, different index orders may lead to completely different encodings, which is also a common problem with one-hot encoding such as encoding ["conv3×3, conv1×1"] into [100,010] or [010,100]. The difference in the final architecture encoding is significant, and this coding method is not suitable for direct input to the predictor model. Arch2vec's variational autoencoder's assumption of Gaussian distribution cannot be guaranteed, and its predictor model has certain defects. DCL-ENAS avoids the adverse impact of directly inputting one-hot encoding into the neural predictor by learning the Manhattan distance from soft encoding to hard encoding, thereby not only improving the reliability of the predictor model during the self-supervised learning stage, but also being able to identify isomorphic graphs. In addition, instead of learning the performance indicators directly, our predictor model learned the relative ranking in a batch, which not only allows the predictor model to learn the minute differences in validation accuracy but also increases the training data to nearly squared times the original data volume, 
{ making better use of a fixed evaluation budget} and improving the reliability of the predictor model.

\begin{table*}[t]
    \centering
    \caption{After conducting architecture search based on validation accuracy, the performance of its optimal architecture on the test set.}
    \resizebox{1\linewidth}{!}{
    \label{tab:test} {
        \input{bench-test}
    }
    }
\end{table*}
\begin{figure}[t]
    \centering
    \includegraphics[width=0.6\textwidth]{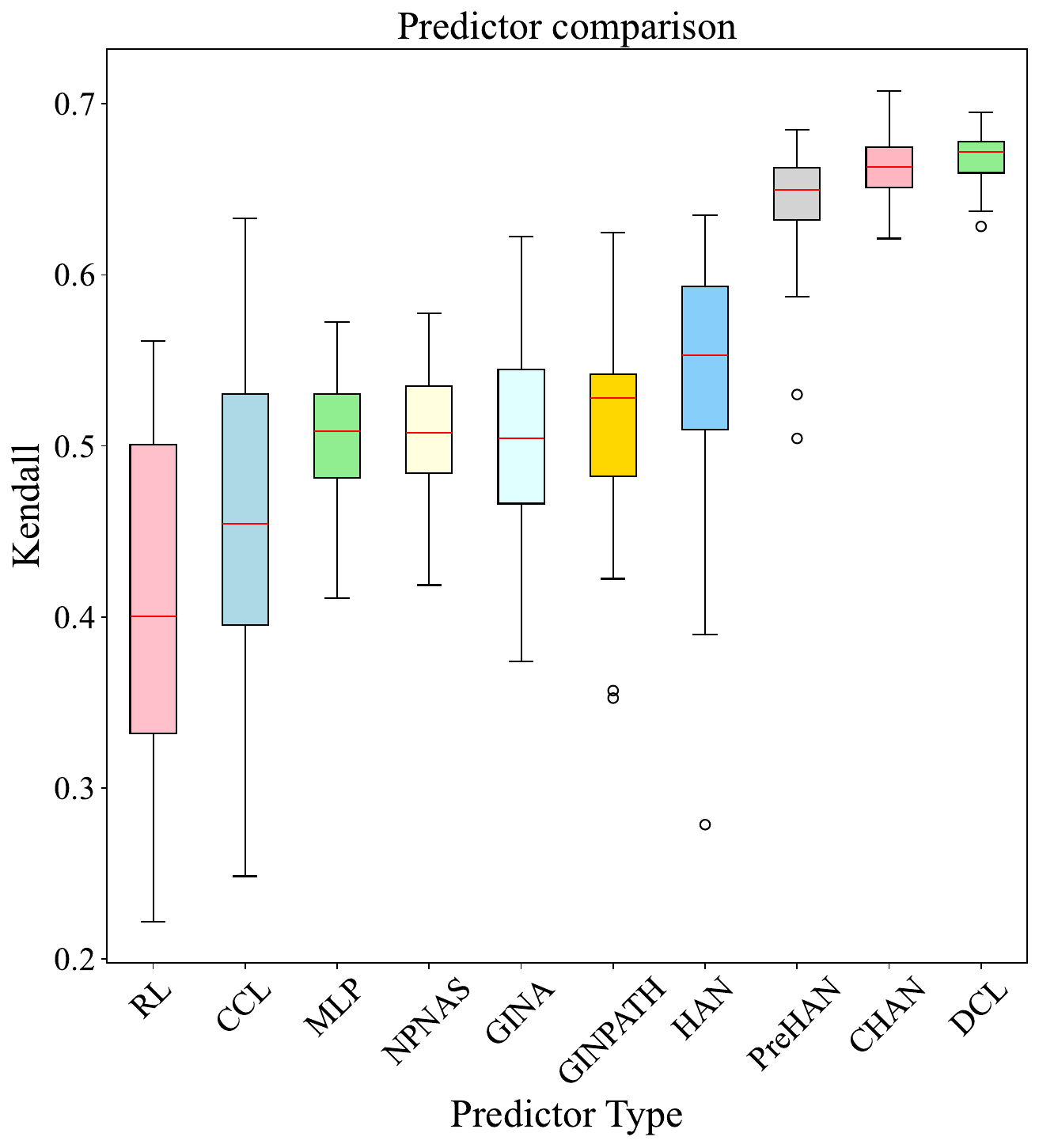}
    \caption{Predictor comparison: Our proposed predictor model HAN outperforms RL, CCL, MLP, and NPENAS on NASBench-201 dataset. HAN shows superior Kendall correlation without pretraining or contrast learning, indicating its advantage in neural network embedding representation. PreHAN, CHAN, and DCL variants demonstrate varying performance levels based on pretraining and fine-tuning strategies, with DCL achieving a Kendall correlation above 0.6.}
    \label{fig:predictor_compare}
\end{figure}
After the algorithm searches for the optimal architecture according to the verification accuracy on different benchmarks, the best architecture's accuracy on the test set is compared in \autoref{tab:test}. Overall, our algorithm is still the best in terms of average ranking. In the test accuracy rate performance on the NASBench-201-ImageNet16-120 benchmark, arch2vec-BO is slightly better. This is because the architecture that performs well on the test set may not necessarily be the best on the validation set, which has been discussed in ~\cite{GCN-predictor} and ~\cite{Npenas}. 

\subsection{In-depth Analysis}
\subsubsection{Predictor comparison}
In this section, we compare the performance of typical predictor models in the above algorithms. We denote the neural predictors commonly used in algorithms such as BANANAS, BOHAMIANN, DNGO, etc., as MLP. We denote the neural predictors in NPENAS-SSRL, NPENAS-SSCCL, and NPENAS-NP algorithms as RL, CCL, and NPENAS, respectively. For fairness, the training set randomly samples 100 data from cifar10-valid of NASBench-201, and the test set is all the remaining data. We have calculated the Kendall-$\tau$ correlation of different predictors on the test set. As shown in \autoref{fig:predictor_compare}, our predictor model HAN (Hierarchical Attention Network) has far exceeded other predictor models in Kendall-$\tau$ correlation without pretraining and contrast learning. This indicates that the predictor model we proposed has an advantage in the representation of neural network embedding. PreHAN is our predictor model after pretraining, but it does not use contrast learning for fine-tuning during fine-tuning. CHAN is a HAN model trained with contrast learning, but it has not been pretrained. DCL is the predictor model in this paper. After pretraining and fine-tuning, the Kendall-$\tau$ correlation is above 0.6.

\subsubsection{{ Effectiveness under limited compute budget}}
  \begin{table}[htbp]
    \centering
    \caption{Kendall's $\tau$ ranking correlation on NASBench-101 and NASBench-201 (CIFAR10-valid) under different
    label (i.e., compute) budgets. For each method, the neural architecture predictor is run independently ten times and the mean $\tau$ is reported; ‘all’ denotes using the entire test set.}
    \label{tab:nb101_nb201_kt}
    \resizebox{0.8\textwidth}{!}{
    \begin{tabular}{c|ccccc|ccccc}
        \hline
        & \multicolumn{5}{c|}{NASBench-101} & \multicolumn{5}{c}{NASBench-201 (CIFAR10-valid)} \\
        \cline{2-11}
        Train label budget & 100 & 172 & 424 & 424 & 4236 & 78 & 156 & 469 & 781 & 1563 \\
        Validation label budget & 200 & 200 & 200 & 200 & 200 & 200 & 200 & 200 & 200 & 200 \\
        Test set size & all & all & 100 & all & all & all & all & all & all & all \\
        \hline
        NPENAS   & 0.391 & 0.545 & 0.710 & 0.679 & 0.769 & 0.343 & 0.413 & 0.584 & 0.634 & 0.646 \\
        NAO      & 0.501 & 0.566 & 0.704 & 0.666 & 0.775 & 0.467 & 0.493 & 0.470 & 0.522 & 0.526 \\
        Arch2Vec & 0.435 & 0.511 & 0.561 & 0.547 & 0.596 & 0.542 & 0.573 & 0.601 & 0.606 & 0.605 \\
        SAENAS-NE& 0.626 & 0.687 & 0.778 & 0.733 & 0.816 & 0.565 & 0.594 & 0.642 & 0.690 & 0.726 \\
        CAP      & 0.656 & 0.709 & 0.791 & 0.758 & 0.831 & 0.600 & 0.684 & 0.755 & 0.776 & 0.815 \\
        \hline
        DCL-ENAS & 0.665 & 0.710 & 0.796 & 0.759 & 0.835 & 0.610 & 0.691 & 0.757 & 0.777 & 0.819 \\
        \hline
    \end{tabular}
    }
\end{table}
To evaluate robustness under { limited compute budgets}, we conduct controlled \emph{label-budget} experiments on NASBench-101 and NASBench-201 (cifar10-valid). Here, a ``label'' refers to { one fully trained architecture–-performance pair used to train the predictor, so the label budget is a direct proxy for compute budget}. For NASBench-101, the training label budgets are set to $\{100,172,424,424,4236\}$; for NASBench-201, they are set to $\{78,156,469,781,1563\}$. In all cases, the validation set size is fixed to $200$, while the test set is either a held-out subset ($100$) or the full set (\texttt{all}), as summarized in Table~\ref{tab:nb101_nb201_kt}. Each method is independently run ten times with different random seeds, and we report the mean Kendall's $\tau$ between the predicted ranking and the ground-truth ranking.

As shown in Table~\ref{tab:nb101_nb201_kt}, DCL-ENAS achieves the best or tied-best Kendall's $\tau$ across all label budgets on both benchmarks. In the most { smallest compute budgets} settings, DCL-ENAS remains competitive: for example, on NASBench-201 with only $78$ labels, it attains $\tau{=}0.610$ (vs.\ the strongest baseline CAP at $0.600$); on NASBench-101 with only $100$ labels, it reaches $\tau{=}0.665$ (vs.\ CAP at $0.656$). As the label budget increases, all methods improve, yet DCL-ENAS consistently maintains a stable advantage. These results demonstrate that our predictor remains highly reliable even under extremely limited supervision, and continues to improve as more labels become available.

We further evaluate the method in full NAS runs by progressively increasing the number of evaluated architectures (i.e., the compute budget) and tracking the best validation accuracy achieved so far. Figure~\ref{fig:convergence} presents results on NASBench-101 and NASBench-201 (CIFAR10-valid, CIFAR100, ImageNet16-120). Across all four benchmarks, DCL-ENAS achieves higher validation accuracy under the same compute budget. The persistent advantage under low and medium budgets indicates that DCL-ENAS can identify high-quality architectures with fewer architecture--performance pairs (i.e., label data).

\begin{figure}[htbp]
    \centering
    \subfloat[NASBench-101\label{fig:nb101_converg}]{
        \includegraphics[width=0.5\textwidth]{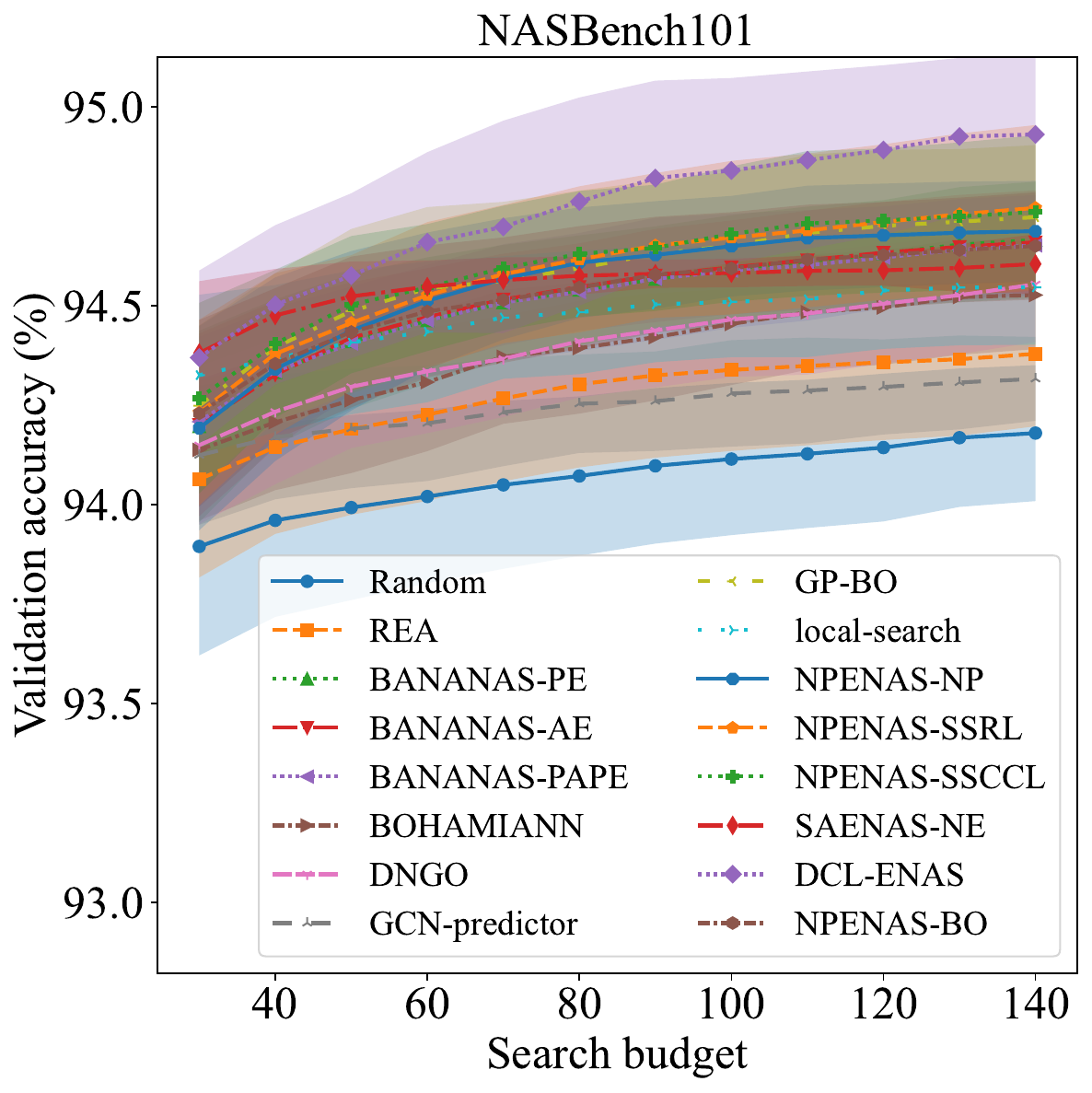}
    }
    \subfloat[NASBench-201-CIFAR10\label{fig:nb201-cifar10_converg}]{
        \includegraphics[width=0.5\textwidth]{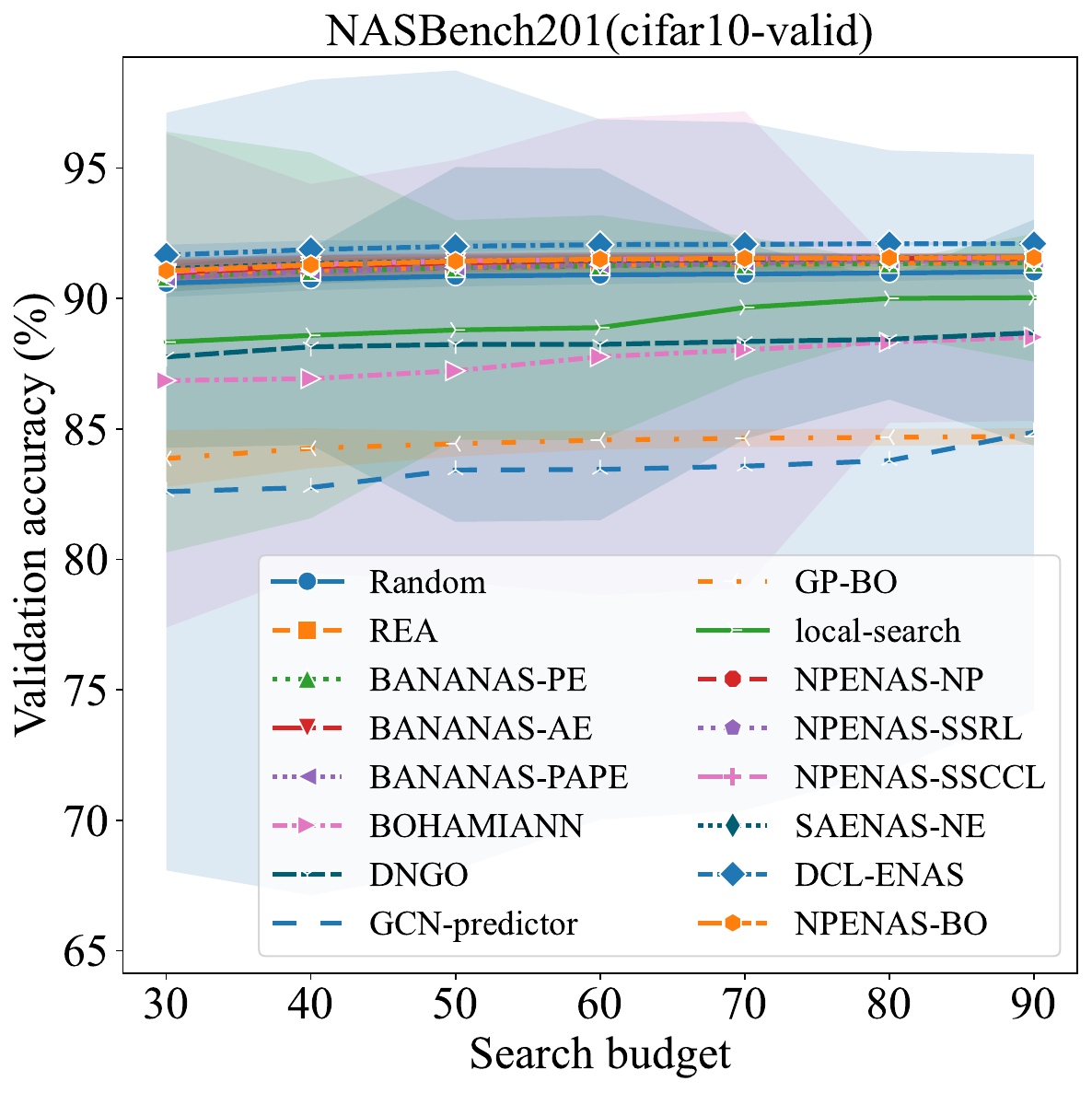}
    }\\
    \subfloat[NASBench-201-CIFAR100\label{fig:nb201-cifar100_converg}]{
        \includegraphics[width=0.5\textwidth]{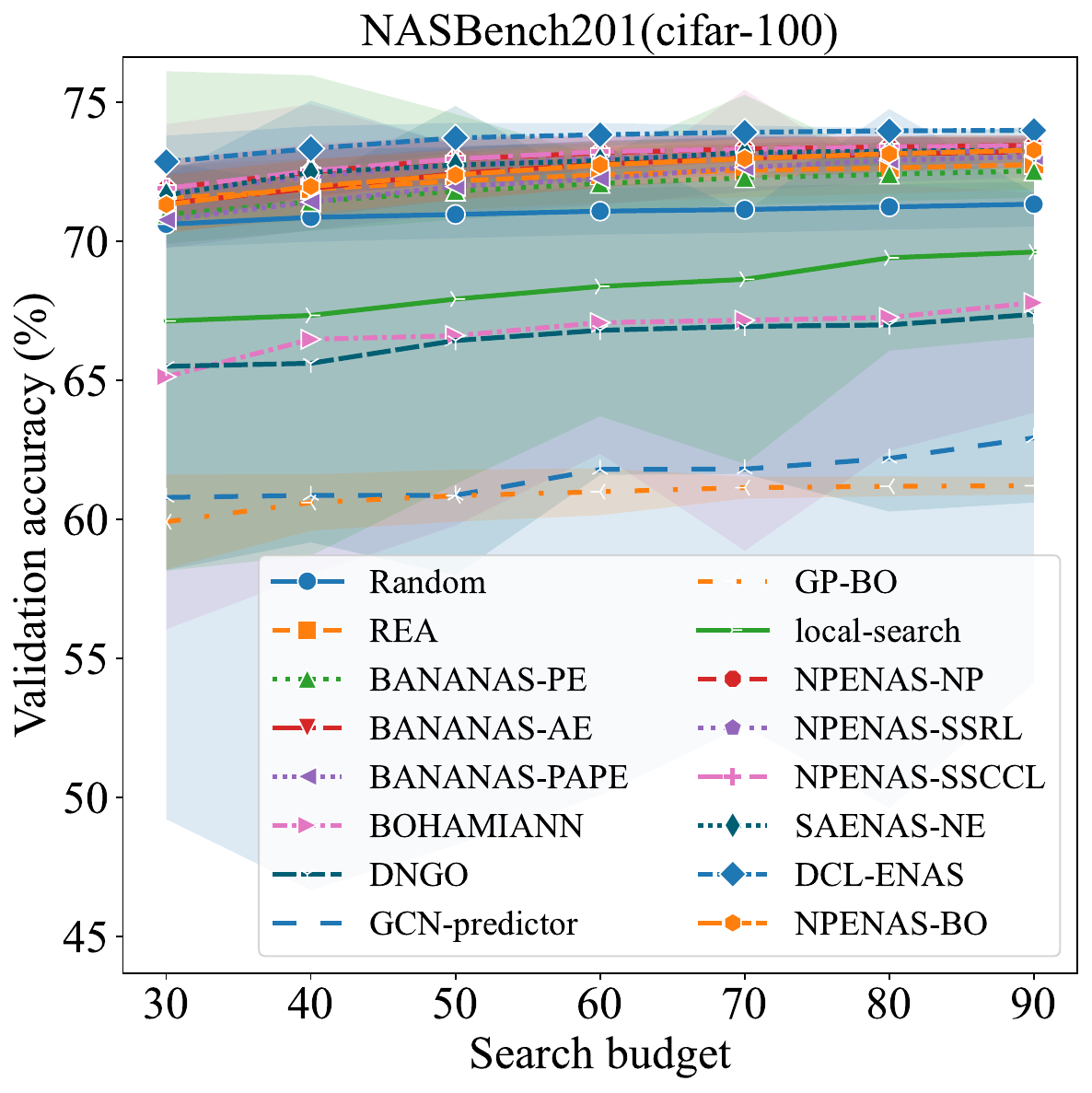}
    }
    \subfloat[{ NASBench-201-ImageNet16-120}\label{fig:nb201-image120_converg}]{
        \includegraphics[width=0.5\textwidth]{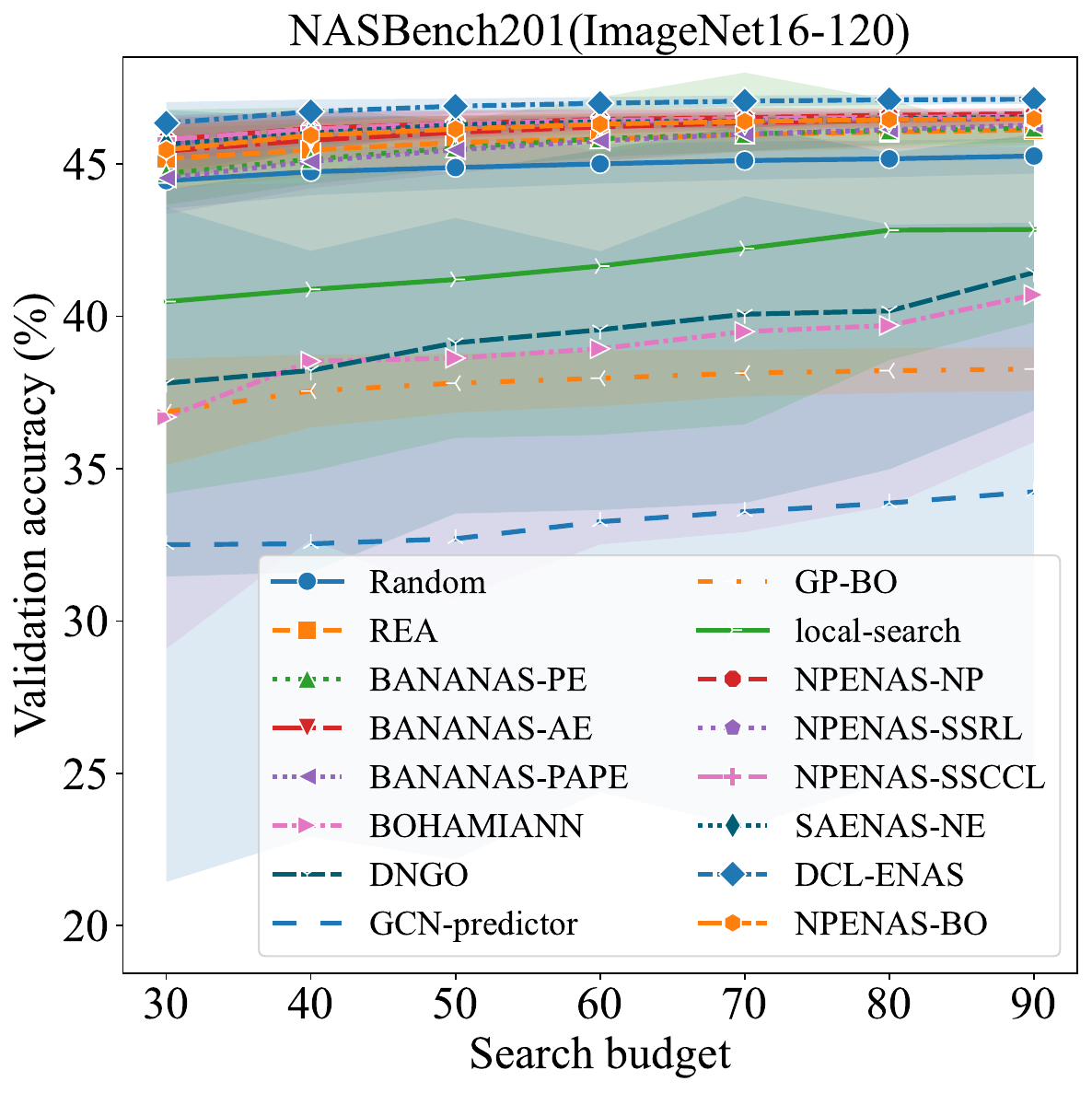}
    }\\
     \caption{
     Validation accuracy curves (Top-1 accuracy on the validation set) of the best architectures obtained by different algorithms on four benchmarks with increasing compute budgets. Since predictor-based NAS algorithms require sampling a subset of architectures to train the predictor before the search starts, the x-axis in the figure does not start from zero.}
    \label{fig:convergence}
\end{figure}

\subsubsection{Ablation Study}
\begin{figure}[htbp]
    \centering
    \subfloat[NASBench-101\label{fig:nb101}]{
        \includegraphics[width=0.5\textwidth]{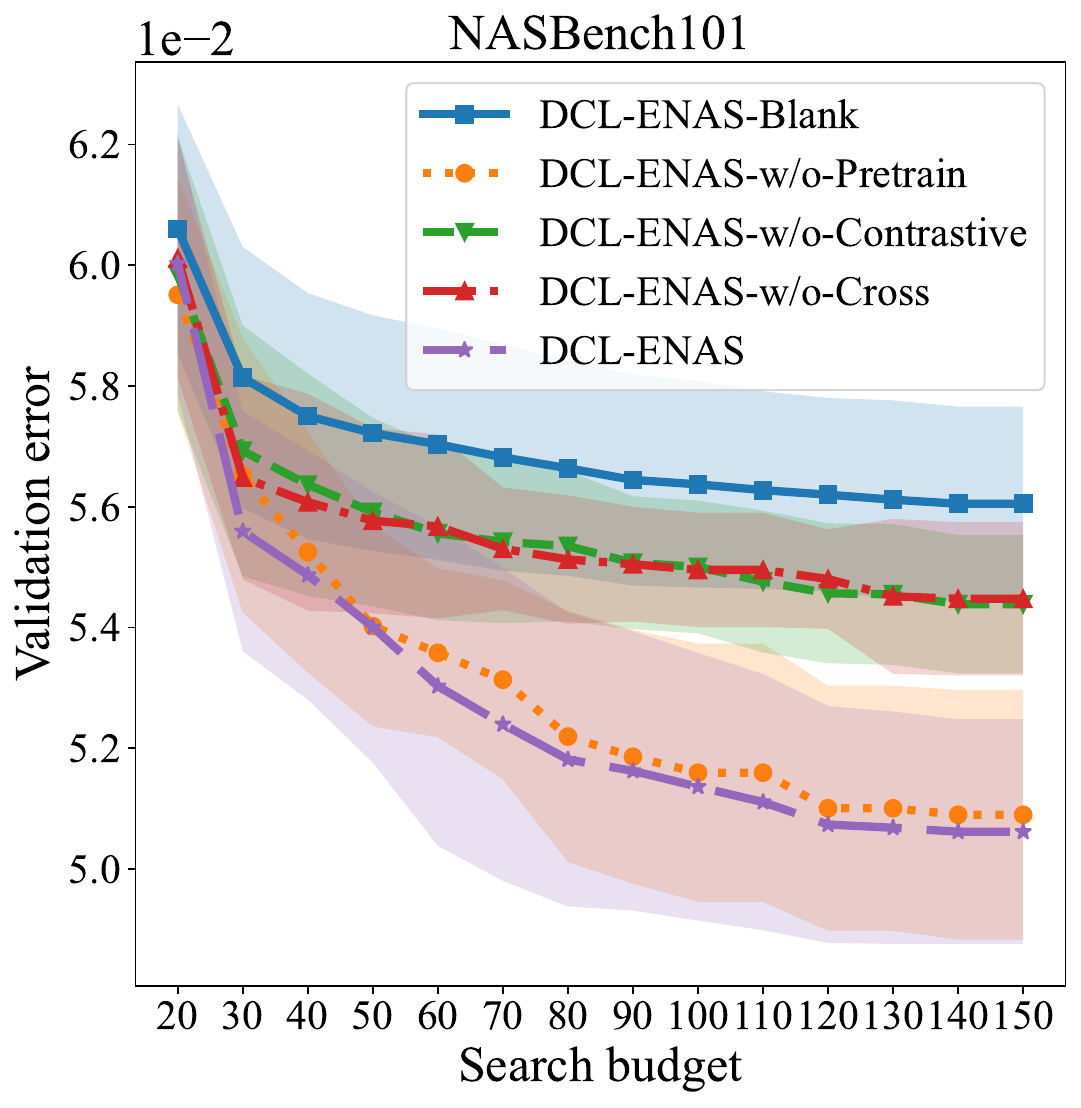}
    }
    \subfloat[NASBench-201-CIFAR10\label{fig:nb201-cifar10}]{
        \includegraphics[width=0.5\textwidth]{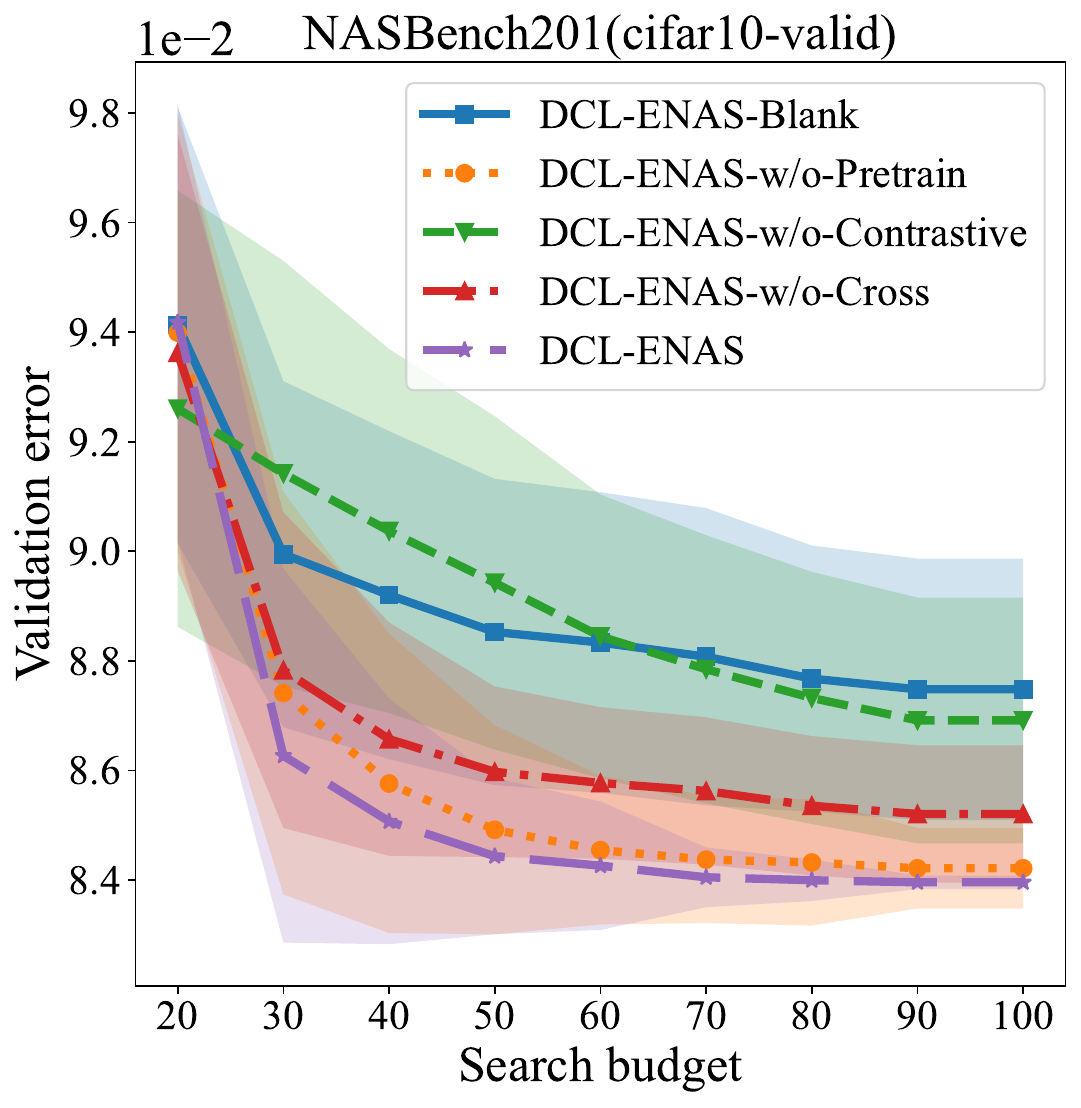}
    }\\
    \subfloat[NASBench-201-CIFAR100\label{fig:nb201-cifar100}]{
        \includegraphics[width=0.5\textwidth]{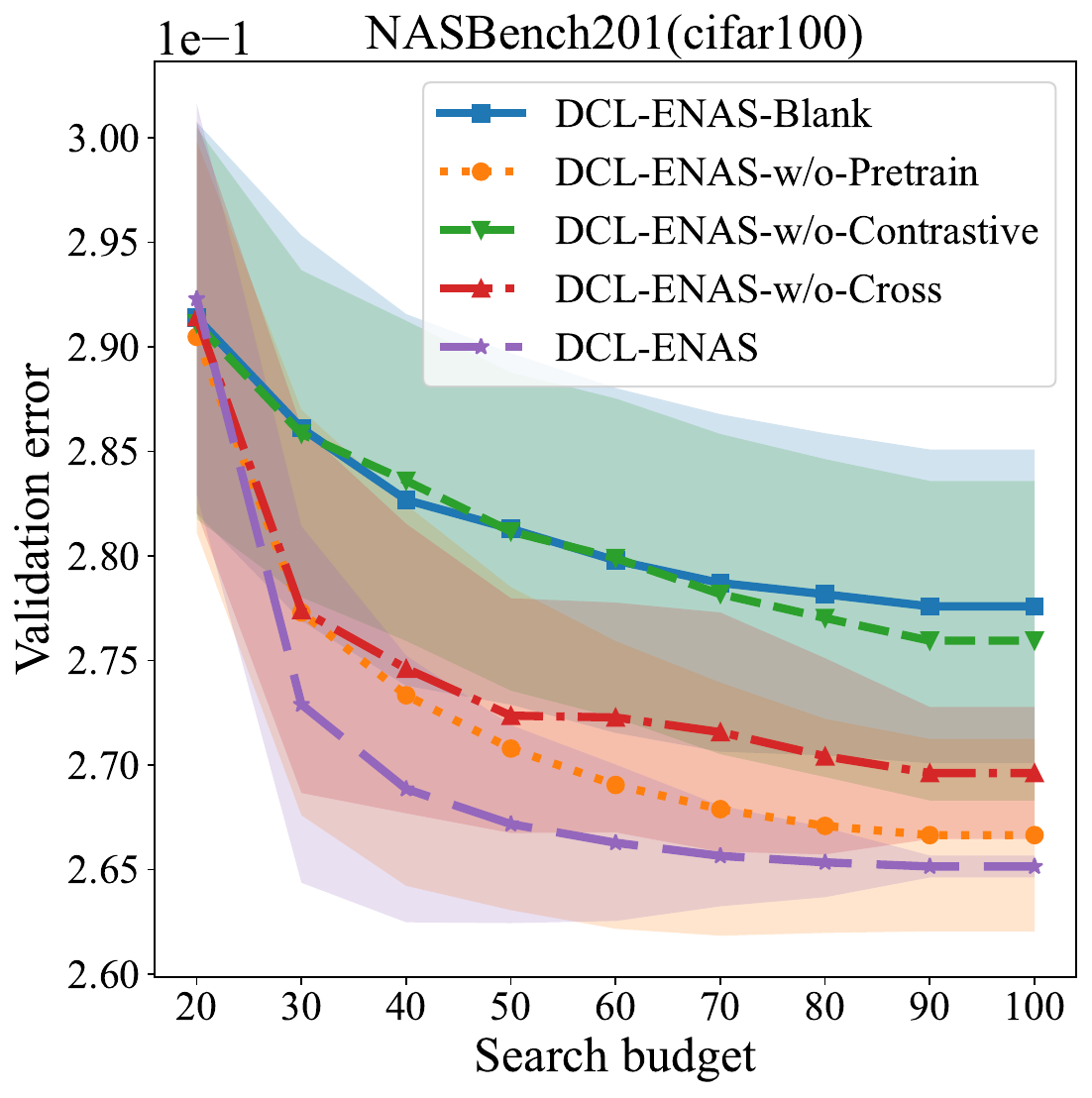}
    }
    \subfloat[NASBench-201-ImageNet16-120\label{fig:nb201-image120}]{
        \includegraphics[width=0.5\textwidth]{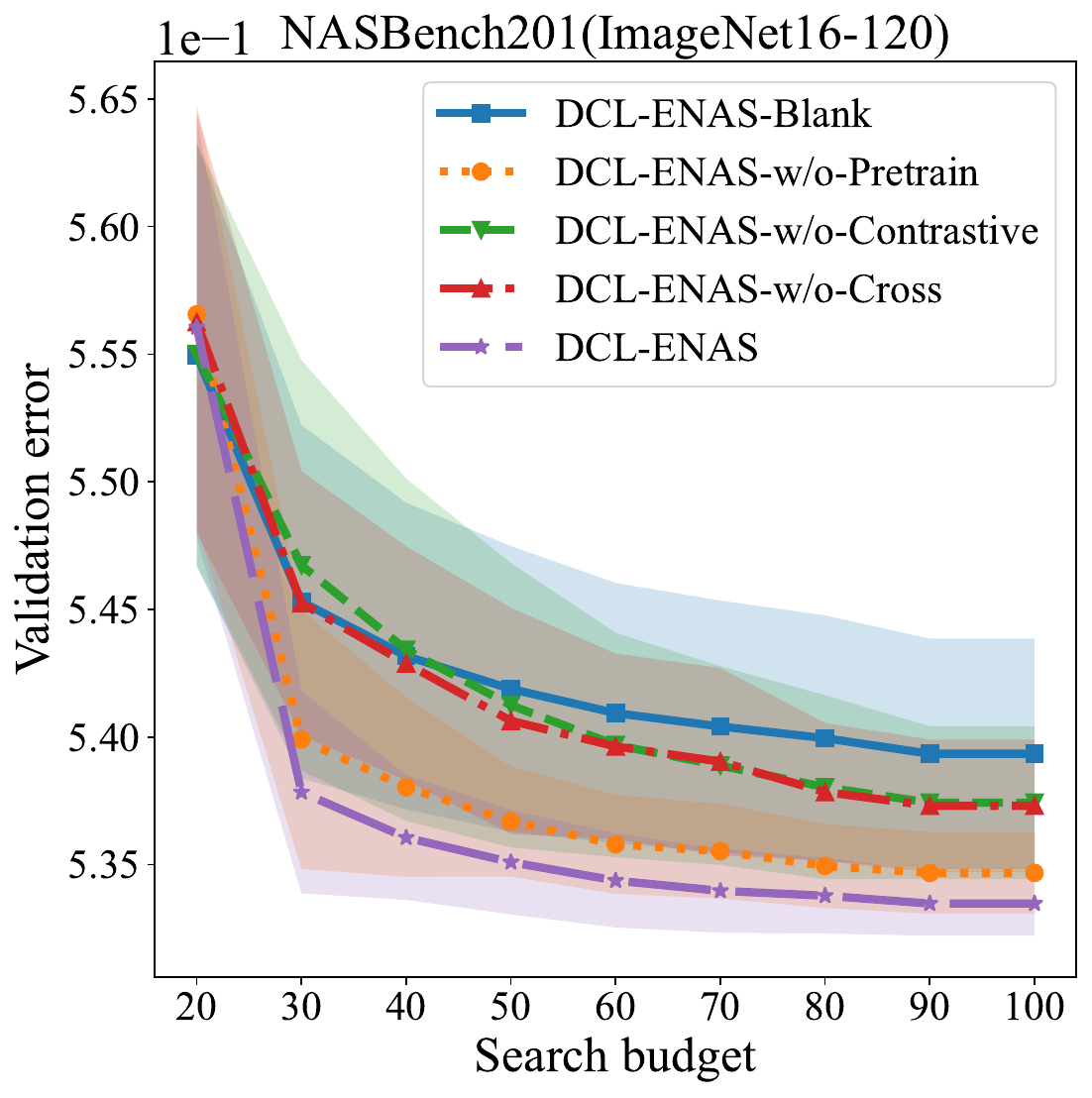}
    }\\
 \caption{Comparison of DCL-ENAS variants across different benchmarks.
    Val Error is the validation classification error of the searched architecture, computed as \(1-\text{validation accuracy}\).}
    \label{fig:ablation}
\end{figure}

In this section, we studied the contributions of the main components proposed in DCL-ENAS. We compared DCL-ENAS with several variants, which are described as follows.
\begin{itemize}
    \item {DCL-ENAS-Blank}: This variant does not have a self-supervised pre-training process. During the search phase, the fine-tuning of the predictor model is a regular regression model, the loss function is MSE, and it is not the ranking-based contrastive loss designed in this paper. In addition, the evolutionary operator does not have a crossover operation.
    \item {DCL-ENAS-w/o-Pretrain}: This variant does not have a self-supervised pre-training process.
    \item {DCL-ENAS-w/o-Contrastive}: This variant replaces the loss during fine-tuning with MSE.
    \item {DCL-ENAS-w/o-Cross}: This variant removes the crossover operator during the evolution process.
\end{itemize}
The experimental results are shown in \autoref{fig:ablation}, and our discussion of the experimental results is as follows:
\begin{itemize}
    \item DCL-ENAS outperforms other variants, indicating that we have learned an effective representation of the neural network architecture using the self-supervised pre-training method. During the evolution process, the neural architecture predictor is fine-tuned in a contrastive learning manner, enhancing the reliability of the predictor model. The crossover operation effectively explores other neural architectures in the search space.
    \item DCL-ENAS-w/o-Pretrain performs weaker than DCL-ENAS, reflecting that the pre-training process can learn the representation of the neural architecture from unlabeled samples, which has a certain significance for improving the reliability of the predictor model.
    \item DCL-ENAS-w/o-Contrastive performs significantly worse than DCL-ENAS, indicating that the reliability of the predictor model is crucial to the entire search process. The predictor model in this paper increases the training data volume to $n(n-1)$ times the original data volume through contrastive fine-tuning, guiding the search process to evolve in the correct direction.
    \item DCL-ENAS-w/o-Cross performs worse than DCL-ENAS, also indicating that the crossover operator has a promoting effect on the evolution process.
    \item DCL-ENAS-Blank performs the worst as expected.
\end{itemize}

\section{Application Study: ECG Time-Series Classification}
\label{sec:application_study}

\begin{figure*}[htbp]
\centering
    \subfloat[Examples of Four ECG Classes. \label{fig:ECG}]{
    \includegraphics[height=0.5\textwidth,keepaspectratio]{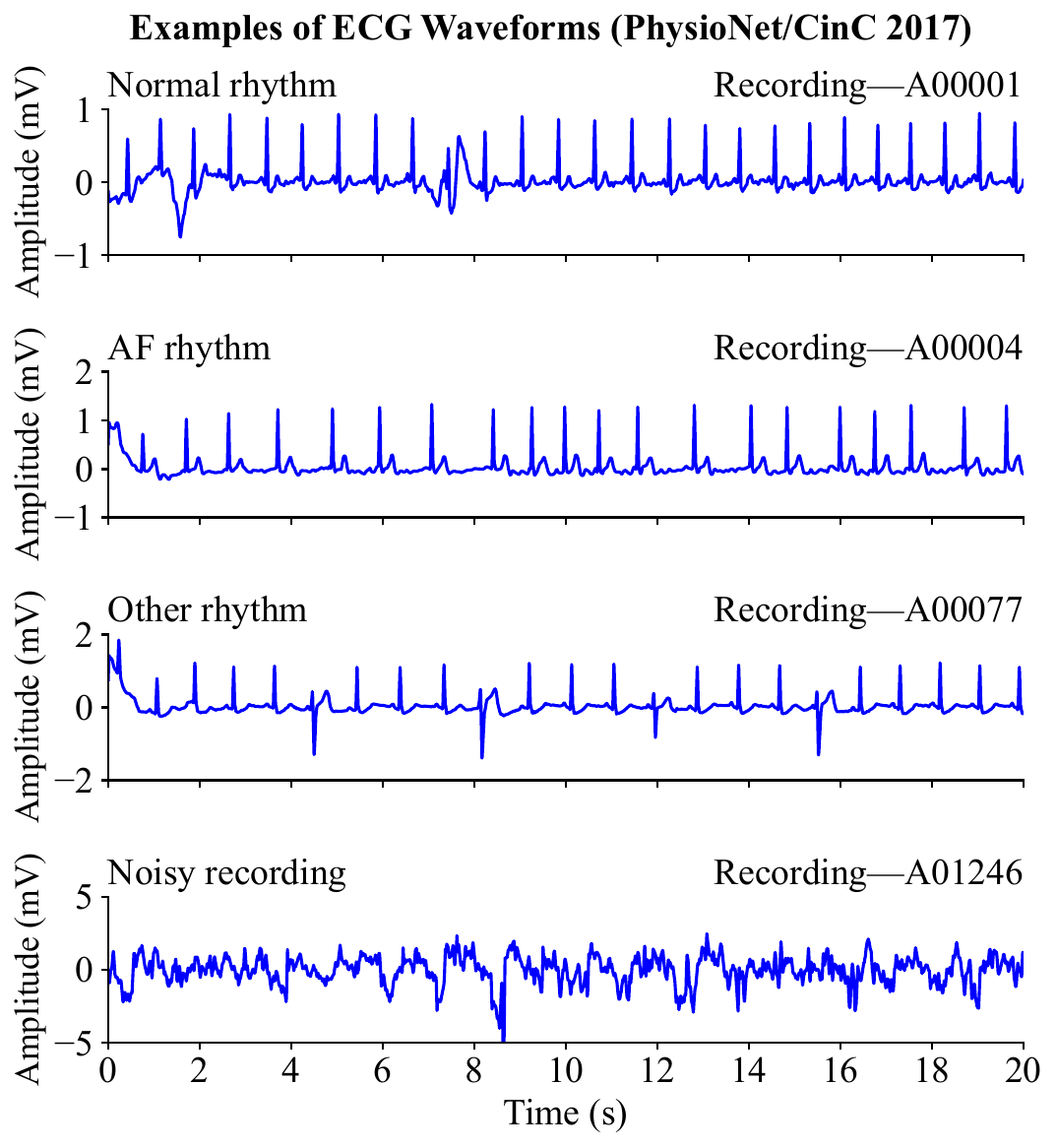}
    }
    \subfloat[Impact of Training Data Reduction on NAS. \label{fig:cut_data}]{
    \includegraphics[height=0.5\textwidth,keepaspectratio]{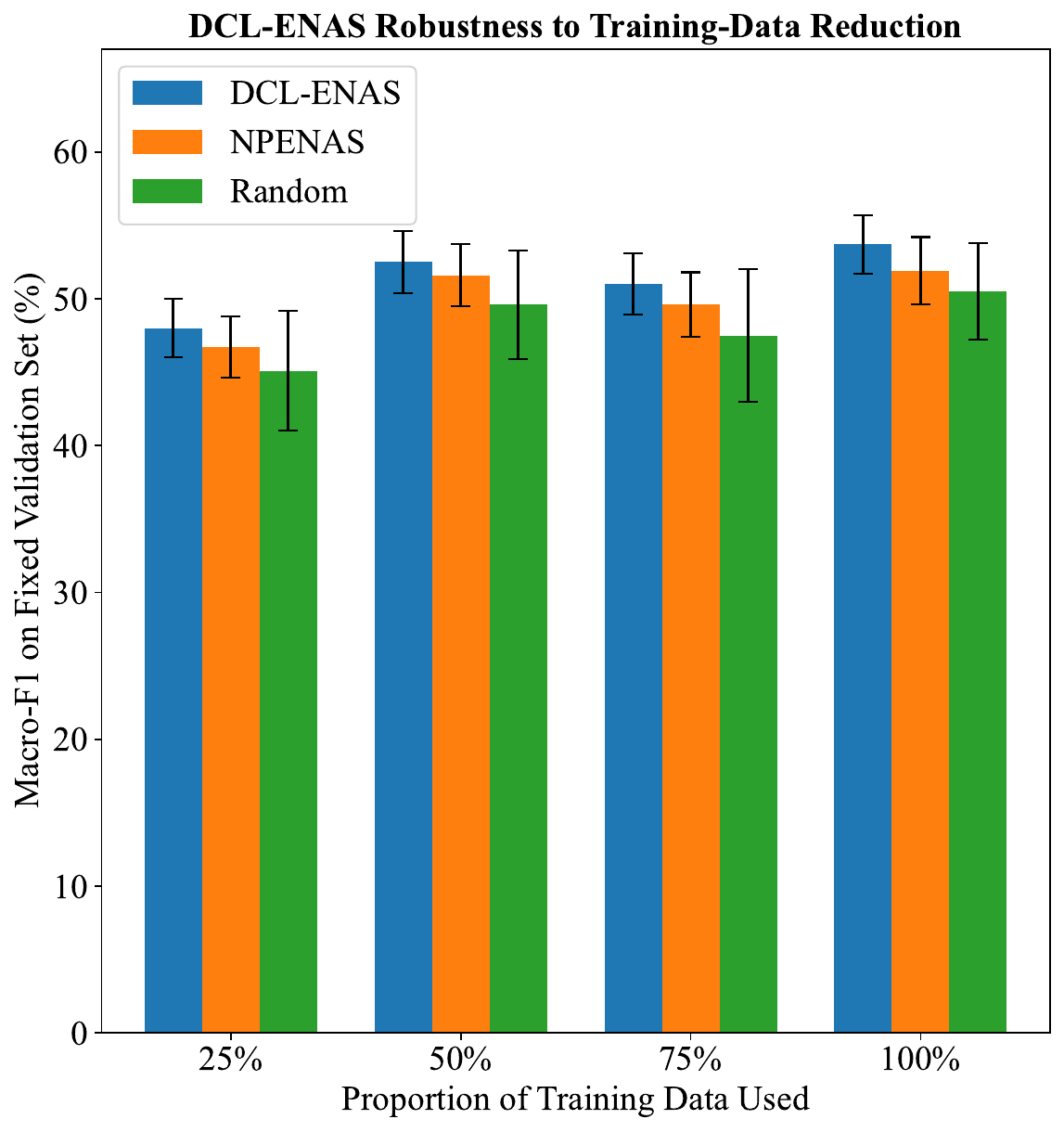}
    }\\
    \subfloat[Discovered Normal Cell. \label{fig:normal_cell}]{
    \includegraphics[width=1\textwidth]{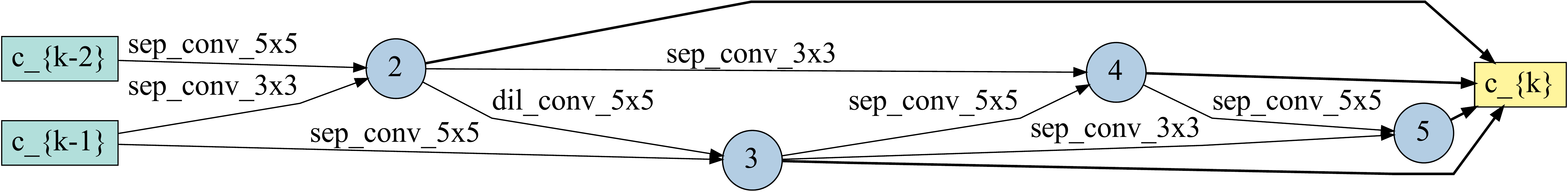}
    }\\
    \subfloat[Discovered Reduction Cell. \label{fig:reduce_cell}]{
    \includegraphics[width=1\textwidth]{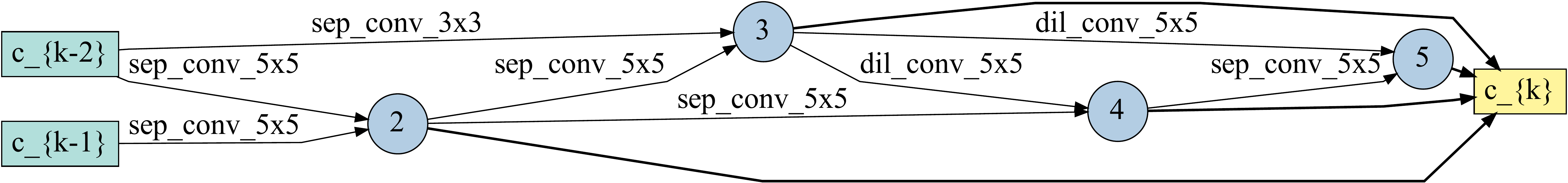}
    }
\caption{Application study: ECG arrhythmia classification using DCL-ENAS.}
\label{fig:application_study}
\end{figure*}

To accommodate the structural characteristics of one-dimensional biomedical time-series data, we manually design a compact yet expressive 1-D CNN neural architecture search space, drawing inspiration from the GAEA PC-DARTS paradigm~\cite{1_D_DARTS}. This search space is tailored to efficiently capture temporal features while maintaining structural parsimony, and is applied to the task of arrhythmia classification on electrocardiogram (ECG) signals.

\subsection{Search Space Design}

Our design adheres to the widely adopted ``Stem + Cells'' paradigm. As illustrated in Table~\ref{tab:1-D}, the input sequence is first projected into a high-dimensional latent space via a single 1-D convolutional layer followed by batch normalization (BN), minimizing information loss in the early stage. Subsequently, the backbone is constructed by stacking eight sequential cells. Reduction Cells are inserted at one-third and two-thirds of the depth to halve the temporal resolution and double the channel width, while the remaining cells are Normal Cells that preserve resolution. The network ends with a global average pooling (GAP) layer followed by a fully connected (FC) classifier, making it suitable for multi-class or multi-label tasks.

By fixing the network depth and downsampling stages, the search is focused on edge connections and primitive selections within cells---factors that significantly impact final performance. Each cell is modeled as a DAG with two input states and four intermediate nodes. Each intermediate node aggregates inputs from its predecessors through edge-level mixed operations. The outputs of the four intermediate nodes are concatenated along the channel dimension to form the cell output. Compared to element-wise addition, concatenation preserves feature diversity and facilitates evolutionary selection pressure.

Each edge operation is selected from a set of eight mutually exclusive primitives: 
\texttt{\{none, \allowbreak max\_pool\_3x3,\allowbreak
avg\_pool\_3x3, \allowbreak skip\_connect, \allowbreak sep\_conv\_3x3, \allowbreak sep\_conv\_5x5,\allowbreak dil\_conv\_3x3, \allowbreak dil\_conv\_5x5\}}.
All convolutional operations are implemented using depthwise separable convolutions to significantly reduce FLOPs. The inclusion of pooling and skip connections offers low-cost transformation paths, enabling flexible control over network depth without excessive parameter stacking.

Both Normal and Reduction Cells follow a 4-step DAG structure, corresponding to intermediate nodes 2 through 5. At step $k$, the new node receives input from $k+2$ preceding states. In Normal Cells, all edges use stride 1 to maintain temporal length. In Reduction Cells, downsampling is achieved by setting the stride of the first two edges in each step to 2, with the remainder at stride 1. This ensures temporal halving and channel doubling via grouped channels and \texttt{FactorizedReduce}. While \cite{1_D_DARTS} targets general architectures, we specifically tailor the design into a one-dimensional CNN search space. The remaining details remain consistent with the original paper.

\begin{table}[htbp]
\centering
\caption{Backbone Configuration Example (Layers = 8), $bn$ = batch size; $Len$ = temporal length.}
\label{tab:1-D}
\resizebox{1\linewidth}{!}{
\begin{tabular}{c|p{0.45\linewidth}|l|c|c}
\hline
Stage & Operation Set & Input $\to$ Output Shape & Searchable & \#Params \\
\hline
Stem & Conv $3{\times}1$ + BN & $(bn, 1, Len)\!\to\!(bn, 48, Len)$ & No & 0.24 k \\
Cell-0 (Normal) & sep\_conv3×3/5×5, dil\_conv 3×3/5×5, max/avg pool3×3, skip, none & $(bn, 48, Len)\!\to\!(bn, 64, Len)$ & Yes & 10 k \\
Cell-1 (Normal) & Same as above & $(bn, 64, Len)\!\to\!(bn, 64, Len)$ & Yes & 10 k \\
Cell-2 (Reduction) & same primitives + stride-2 subsample & $(bn, 64, Len)\!\to\!(bn, 128, Len/2)$ & Yes & 30 k \\
Cell-3 (Normal) & Same as above & $(bn, 128, Len/2)\!\to\!(bn, 128, Len/2)$ & Yes & 28 k \\
Cell-4 (Normal) & Same as above & $(bn, 128, Len/2)\!\to\!(bn, 128, Len/2)$ & Yes & 28 k \\
Cell-5 (Reduction) & same primitives + stride-2 subsample & $(bn, 128, Len/2)\!\to\!(bn, 256, Len/4)$ & Yes & 38 k \\
Cell-6 (Normal) & Same as above & $(bn, 256, Len/4)\!\to\!(bn, 256, Len/4)$ & Yes & 36 k \\
Cell-7 (Normal) & Same as above & $(bn, 256, Len/4)\!\to\!(bn, 256, Len/4)$ & Yes & 28 k \\
GAP + FC & Adaptive AvgPool + Linear & $(bn, 256, 1)\!\to\!(bn, 4)$ & No & 1 k \\
\hline
\end{tabular}}
\end{table}

When porting DCL\mbox{-}ENAS from the single-cell spaces of NASBench-101/201 to the dual-cell 1-D DARTS space, we keep the entire two-stage workflow—\emph{CLP pre-training followed by CLF-ENAS search}—as well as the original contrastive losses and evaluation budget.  The only structural change is that each architecture now contains a Normal Cell and a Reduction Cell; their adjacency matrices are placed on the diagonal of a block-diagonal matrix:
      $$
        \mathbf{A}=\begin{bmatrix}
                     \mathbf{A}_{\text{Norm}} & \mathbf{0}\\
                     \mathbf{0} & \mathbf{A}_{\text{Redu}}
                   \end{bmatrix},
      $$
and every node feature receives an extra one-hot \texttt{Cell-ID}.  This allows the GIN encoder to process both cells as one enlarged graph, with the sole consequence of a wider input dimension.  For path encoding, we concatenate \texttt{Cell-ID} and \texttt{Path-ID} to create a global index $\text{GlobalID}=\texttt{Cell-ID}\!\times\!P+\texttt{Path-ID}$. Here, \(P\) denotes the total number of unique (non-duplicate) paths available to a \emph{single} cell in the search space (after deduplication of path equivalences). For example, if the Normal Cell contributes \(P=1000\) unique paths, the Reduction Cell adopts an index offset of \(P\) and its path numbering starts at \(P+1=1001\); equivalently, the Normal Cell occupies indices \([1,1000]\) and the Reduction Cell \([1001,2000]\). The path table therefore grows from roughly $10^{3}$ entries to about $2{\times}10^{4}$, but the one-hot-per-path rule, fixed template and Manhattan-distance permutation invariance remain fully valid, and the Transformer path-attention simply doubles its maximum sequence length from $L_{\text{seq}}$ to $2L_{\text{seq}}$.
The evolutionary operators are unchanged except that crossover is applied within the same cell type to preserve tensor compatibility, so all hyper-parameters $(P_{\mathrm c},P_{\mathrm m},P_{\mathrm{keep}})$ are reused verbatim.  Finally, predictor fine-tuning still relies on the pairwise ranking loss; the only difference is that the target metric switches from \emph{Validation Accuracy} on image data to \emph{Macro-F1} on ECG data, a change that affects the labels but not the loss formulation.  In short, aside from expanding the adjacency matrix and doubling the path table, every core component and contrastive-learning strategy of DCL-ENAS carries over unchanged to the 1-D DARTS setting.%

\subsection{Experimental Setup and Data Augmentation}

We evaluate DCL-ENAS on the public PhysioNet 2017 Challenge dataset~\cite{ECG}, which involves a four-class classification task using single-lead ECG recordings. The dataset contains 8,528 records with durations ranging from 9 to 61 seconds (mean: 32.5s, std: 10.9s) sampled at 300Hz using AliveCor Kardia devices. Each record is manually annotated by cardiologists into one of four mutually exclusive categories: normal rhythm ($n$ = 5,154), atrial fibrillation (AF)($n$ = 771), other rhythm ($n$ = 2,557), and noisy signals ($n$ = 46). Here, $n$ denotes the number of records in each category. Examples are shown in Fig.~\ref{fig:ECG}.

To mitigate class imbalance, we apply a sliding-window augmentation strategy. With a default window size of 1,000 and stride of 500, normal samples are split with 50\% overlap. For AF, other rhythm, and noise, strides are reduced by approximately 1/6, 1/2, and 1/20 respectively, yielding more overlapping fragments and improved sample density. A total of 150,000 samples are randomly selected for training.

\subsection{Experimental Results}

During NAS, we keep the outer evolutionary process and the fine-tuning of the neural predictor unchanged. However, we intentionally reduce the ECG training data used in the inner loop that computes validation accuracy. Four data volume ratios are tested: 25\%, 50\%, 75\%, and 100\%. The validation set is fixed across all settings, so performance differences are solely attributable to training data availability. Each setting is repeated 10 times with different seeds, and results are reported as mean ± standard deviation of macro F1 score.

As shown in Fig.~\ref{fig:cut_data}, reducing the training data in the inner loop has limited impact on NAS performance under a budget of 50 evaluations. In line with prior work~\cite{Less_Is_More,Data_Pruning,Subset_Selection}, even extreme data reduction (e.g., 20\%) incurs less than 3\% F1 degradation. Surprisingly, subsets like 50\% sometimes yield architectures that outperform the full dataset. 
These findings collectively support the ``performance ranking assumption'', which posits that low-fidelity~\cite{Multi-Fidelity} evaluation using data subsets is often sufficient to preserve the relative ranking of architectures. This assumption is particularly advantageous in scenarios where surrogate models rely on ordinal information rather than precise fitness values.

We report end-to-end NAS wall-clock time in GPU-days on a single NVIDIA GeForce RTX 4090 (24 GB) to reflect the practical cost users would face. Unless otherwise specified, all runs in this section use a maximum evaluation budget of 50 architectures. When obtaining the fitness (F1) of an architecture, we train for 50 epochs. Table~\ref{tab:runtime} summarizes DCL-ENAS runtime across four training-set ratios (25\%, 50\%, 75\%, and 100\%). Under identical settings, NPENAS is typically 0.5--1.5 hours faster than DCL-ENAS, and Random search is 0.5--2.0 hours faster. Despite the small differences in wall-clock time, DCL-ENAS consistently discovers the best architectures on our ECG task, with macro-F1 $\geq 0.55$(Fig.~\ref{fig:cut_data}). 

\begin{table}[t]
\centering
\caption{
End-to-end runtime (GPU-days on a single RTX~4090). All methods use an evaluation budget of 50.}
\label{tab:runtime}
\resizebox{0.5\linewidth}{!}{
\begin{tabular}{|l|c|c|c|}
\hline
{Training-set ratio} & {DCL-ENAS} & {NPENAS} & {Random} \\
\hline
25\%  & 1.625 & 1.583 & 1.604 \\ \hline
50\%  & 2.500 & 2.450 & 2.417 \\  \hline
75\%  & 5.000 & 4.967 & 4.958 \\   \hline
100\% & 7.833 & 7.771 & 7.758 \\    
\hline
\end{tabular}}
\end{table}

\subsection{Final Architecture Visualization}

The final architectures searched on ECG data using DCL-ENAS are visualized in Fig.~\ref{fig:normal_cell} (Normal Cell) and Fig.~\ref{fig:reduce_cell} (Reduction Cell), tailored specifically for time-series feature extraction in one-dimensional biomedical signals.

\section{Conclusion}
\label{sec:conclusion}
This article proposes a novel predictor-assisted evolutionary neural architecture search method, DCL-ENAS, which enhances the performance of the predictor model through Dual Contrastive Learning. This proposed method addresses the limitations of traditional ENAS methods, which usually require large amounts of computational resources due to the necessity of training each architecture from scratch. In the contrastive pretraining stage of DCL-ENAS, we reexamine the neural architecture from the perspective of information flow and perform representation learning based on the similarity between different semantic transformations of the same architecture in recognition batches. Specifically, the hard encoder is used to generate understandable information flow vector representations and group similar architectures within the search space. Afterwards, as part of the predictor model, the soft encoder mimics the grouping knowledge of the information flow vectors from the hard encoder. This pre-training phase takes advantage of contrastive self-supervised learning, eliminating the need for any labeled data on architectural performance and greatly reducing computational costs. In the contrastive fine-tuning stage, the predictor model is fine-tuned to rank different architectures. It effectively extends the training dataset to \( n(n-1) \) pairs, significantly alleviating the reliability issues caused by the limited { compute budget} available in traditional predictor-assisted evolutionary neural architecture search methods. We have also developed neural network information flow crossover operators, which combine path information between parent individuals, enhance population diversity, and guide the evolution towards more promising solutions. DCL-ENAS achieves good performance on two different NASBench search spaces. Experiments show that the proposed paradigm of self-supervised pre-training and contrastive fine-tuning can effectively improve the reliability of the predictor model. 
DCL-ENAS secures first place on both NASBench tracks, raising validation accuracy
by 0.05\%--0.39\% with respect to the strongest baselines (average rank = 1.00).
Under a 50-evaluation budget for ECG arrhythmia classification, it delivers a
2.5-percentage-point accuracy gain over a non-NAS baseline (Random), underscoring the robustness of the
two-stage contrastive predictor in noisy and highly imbalanced biomedical time-series settings.

Despite the strong performance of DCL-ENAS across multiple benchmark tasks, several limitations remain. First, the contrastive pretraining of the predictor’s soft encoder introduces a non-negligible computational overhead, particularly during the self-supervised training phase used to initialize its parameters. As the architecture search space grows exponentially, this overhead may become a bottleneck that hinders the scalability of the method. Moreover, this work is primarily restricted to cell-based search spaces, and the generalization capability of DCL-ENAS to other domains—such as large language model (LLM) architecture optimization—has not yet been empirically validated. Future research should aim to enhance the efficiency of contrastive pretraining and extend the ENAS framework to a broader range of tasks and search spaces, with particular emphasis on its potential for optimizing LLM architectures.{ In particular, given that our framework is compatible with multi-objective NAS (e.g., using Pareto-based selection), future work will jointly optimize accuracy, parameter count, and inference FLOPs under the same compute budget constraint.}

% \section*{References}
% \bibliographystyle{elsarticle}

\bibliography{DCL-ENAS_main}

\end{document}

%% file: DCL-ENAS.tex
\begin{algorithmic}[1]

\REQUIRE~ $\mathcal{A}$: Architecture search space; $N$:Size of population; $r$: Number of offspring members associated with each reference solution;
$fes_{\text{\text{max}}}$: The maximum number of the real evaluations;~\\ 
\ENSURE The best neural network architecture.

\STATE $W^{\text{pre}} \leftarrow$ \textbf{ContrastivePretrain}($\mathcal{A}$) \COMMENT{Pretrain soft encoder on unlabeled architectures (CLP stage).} 
\STATE $D_{\text{init}}$ = $\{a_1, a_2, ..., a_N\}$ \COMMENT{Sample $N$ architectures and query their real validation accuracies. Each architecture $\alpha_i$ and its corresponding validation accuracy $ACC_i^{eval}$ constitute one training sample for the predictor, denoted as $a_i = (\alpha_i, ACC_i^{eval})$.} 
\STATE $D_{\text{total}}$ $\gets$ $D_{\text{init}}$ 
\STATE $fes$ = $|D_{\text{init}}|$
\STATE $M \leftarrow$ \textbf{InitializePredictor}($W^{\text{pre}}$) \COMMENT{Load predictor model $M$ with weights $W^{pre}$.}
\STATE $M \leftarrow$ \textbf{FineTune}($M$, $D_{\text{init}}$) \COMMENT{Fine-tune predictor $M$ on labeled data $D_{\text{init}}$.}
\STATE $P_{\text{label}}$ $\gets$ $D_{\text{init}}$ \COMMENT{$P_{\text{label}}$ denotes the set of all architectures with ground-truth validation accuracy.}
\WHILE{$fes$ $<$ $fes_{\text{max}}$}
     \STATE $t$ = 1, $P_t$ = $P_{\text{label}}$;
     \STATE \COMMENT{Lines 11\textasciitilde16: Run evolution for $t_{\text{gap}}$ generations before querying real evaluations.}
     \WHILE{$t$ $<$ $t_{\text{gap}}$}
        \STATE $Q_t \leftarrow$ \textbf{GenerateOffspring}($P_t$, $r$) \COMMENT{Generate offspring via crossover and mutation.}
        \STATE $R_t \leftarrow P_t \cup Q_t$
        \STATE $P_{t+1} \leftarrow$ \textbf{EnvSelection}($R_t$, $N$, $M$) \COMMENT{Use predictor $M$ to rank and select top-$N$ for next generation.} 
        \STATE $t \leftarrow t + 1$
     \ENDWHILE
    \STATE $P_{\text{infill}} \leftarrow$ \textbf{InfillSampling}($P_{t_{\text{gap}}}$, $N_{\text{infill}}$, $M$, $r$) \COMMENT{Select subset for real evaluation based on $M$’s prediction and uncertainty.}

    \STATE $P_{\text{infill}}$ $\leftarrow$ \textbf{QueryRealAcc}($P_{\text{infill}}$) \COMMENT{Train on training set and evaluate on validation set to obtain true accuracy.}
    \STATE $fes \leftarrow fes + |P_{\text{infill}}|$
    \STATE $D_{\text{total}} \leftarrow D_{\text{total}} \cup P_{\text{infill}}$
    \STATE $M \leftarrow$ \textbf{FineTune}($M$, $D_{\text{total}}$)
    \STATE $P_{\text{label}} \leftarrow$ \textbf{TopSelect}($P_{\text{label}} \cup P_{\text{infill}}$, $N$)
\ENDWHILE
\STATE $P_{\text{best}} \leftarrow \arg\max_{a \in D_{\text{total}}} \text{ValAcc}(a)$
\RETURN  $P_{\text{best}}$
\end{algorithmic}

%% file: parameter.tex
\begin{tabular}{|c|c|c|c|}
\hline
% \multicolumn{2}{|c|}{NASBench-101} & \multicolumn{2}{c|}{NASBench-201} \\
% \hline
\multicolumn{4}{|c|}{Train Predictor} \\
\hline
Module  & Parameter & NASBench-101 & NASBench-201 \\ 
\hline
\multirow{4}{*}{Pretrain/Finetune} & batch size & 20000/8192 & 10000/8192 \\ 
\cline{2-4} 
& epoch & 200/50 & 200/50  \\ 
\cline{2-4} 
& cluster sizes & [5,10,20] & $ \{x \in \mathbb{Z} | 10 \leq x \leq 20\} $ \\ 
\cline{2-4} 
& $fes_{max}$ & 150 & 100 \\ 
\hline

% \multirow{3}{*}{Finetune} & batch size & 8192 & 8192  \\ 
% \cline{2-4} 
% & epoch & 50 & 50 \\ 
% \cline{2-4} 
% & $fes_{max}$ & 150 & 100 \\ 
% \hline
\multirow{8}{*}{Surrogate model}& Learning rate & 0.001 & 0.001 \\ 
\cline{2-4} 
& $GIN$ Embedding dim & 6 & 8 \\ 
\cline{2-4} 
& Max token count $L_{seq}$ & 8 & 5 \\ 
\cline{2-4} 
& Hard encode length & 120 & 96 \\ 
\cline{2-4} 
& Soft encode length & 32 & 32 \\ 
\cline{2-4} 
& $d_{FM}$ & 16 & 16 \\ 
\cline{2-4} 
& Soft‑feature dim $D$ & 16 & 16 \\ 
\cline{2-4} 
& Token‑embed dim $D_e$ & 128 & 128 \\ 
\hline
\multicolumn{4}{|c|}{Evolutionary and Search} \\
\hline
\multirow{8}{*}{Evolution}& The size of population $N$& \multicolumn{2}{c|}{20} \\ \cline{2-4} 
& Crossover probability $P_c$ & \multicolumn{2}{c|}{0.9} \\ \cline{2-4}
&Mutation probability $P_m$ & \multicolumn{2}{c|}{0.1} \\ \cline{2-4}
&Node‑keep probability $P_{\text{keep}}$ & \multicolumn{2}{c|}{0.5} \\ \cline{2-4}
&Offspring ratio $r$ & \multicolumn{2}{c|}{6} \\ \cline{2-4}
&Evaluation interval $t_{\text{gap}}$ & \multicolumn{2}{c|}{5} \\ \cline{2-4}
&Family radius $C_a$ & \multicolumn{2}{c|}{6} \\ \cline{2-4}
&$\!$Infill size $N_{\text{infill}}$ & \multicolumn{2}{c|}{10} \\ \cline{2-4}
% & Soft‑feature dim$D$ & \multicolumn{2}{c|}{16} \\ \cline{2-4}
% & Token‑embed dim$D_e$ & \multicolumn{2}{c|}{128} \\ 
\hline
\end{tabular}

%% file: valid.tex
\begin{tabular}{|c|c|c|c|c|c|c|c|}
\hline Benchmark & NASBench-101 & NASBench-201-cifar10-valid & NASBench-201-cifar100 & NASBench-201-ImageNet16-120 & +/$\approx$/- & Average Rank & Optimized type \\ \hline
Budget & 150 & 100 & 100 & 100 & - & - & - \\ \hline
Random & \text{94.19$\pm$0.16(+)(19)} & \text{90.98$\pm$0.25(+)(19)} & \text{71.36$\pm$0.61(+)(19)} & \text{45.36$\pm$0.60(+)(18)} & 4/0/0 & 18.75 & R \\ \hline
REA & \text{94.39$\pm$0.19(+)(17)} & \text{91.40$\pm$0.27(+)(14)} & \text{72.86$\pm$0.83(+)(14)} & \text{46.07$\pm$0.54(+)(17)} & 4/0/0 & 15.50 & EA \\ \hline
BANANAS-PE & \text{94.64$\pm$0.11(+)(9)} & \text{91.33$\pm$0.36($\approx$)(17)} & \text{72.49$\pm$0.97(+)(17)} & \text{46.21$\pm$0.46(+)(14)} & 3/1/0 & 14.25 & BO \\ \hline
BANANAS-AE & \text{94.67$\pm$0.13(+)(5)} & \text{91.53$\pm$0.17(+)(8)} & \text{73.21$\pm$0.62(+)(10)} & \text{46.56$\pm$0.15(+)(4)} & 4/0/0 & 6.75 & BO \\ \hline
BANANAS-PAPE & \text{94.67$\pm$0.13(+)(6)} & \text{91.50$\pm$0.20($\approx$)(12)} & \text{73.21$\pm$0.62(+)(9)} & \text{46.39$\pm$0.28(+)(10)} & 3/1/0 & 9.25 & BO \\ \hline
BOHAMIANN & \text{94.53$\pm$0.16(+)(14)} & \text{91.37$\pm$0.20(+)(16)} & \text{72.66$\pm$0.71(+)(15)} & \text{46.11$\pm$0.46(+)(16)} & 4/0/0 & 15.25 & BO \\ \hline
BOGCN-NAS & \text{94.62$\pm$0.15(+)(10)} & \text{91.53$\pm$0.13(+)(9)} & \text{73.30$\pm$0.39(+)(8)} & \text{46.47$\pm$0.34(+)(9)} & 4/0/0 & 9.00 & BO \\ \hline
DNGO & \text{94.54$\pm$0.21(+)(13)} & \text{91.38$\pm$0.18(+)(15)} & \text{72.60$\pm$0.66(+)(16)} & \text{46.19$\pm$0.47(+)(15)} & 4/0/0 & 14.75 & BO \\ \hline
GCN-predictor & \text{94.30$\pm$0.13(+)(18)} & \text{91.02$\pm$0.28(+)(18)} & \text{71.53$\pm$0.72(+)(18)} & \text{45.28$\pm$0.62(+)(19)} & 4/0/0 & 18.25 & R \\ \hline
GP-BO & \text{94.72$\pm$0.22(+)(3)} & \text{91.52$\pm$0.17(+)(11)} & \text{73.35$\pm$0.45(+)(6)} & \text{46.52$\pm$0.44($\approx$)(6)} & 3/1/0 & 6.50 & BO \\ \hline
local-search & \text{94.54$\pm$0.17(+)(12)} & \text{91.52$\pm$0.15(+)(10)} & \text{73.02$\pm$0.77(+)(12)} & \text{46.52$\pm$0.27(+)(7)} & 4/0/0 & 10.25 & - \\ \hline
arch2vec-RL & \text{94.42$\pm$0.18(+)(16)} & \text{91.44$\pm$0.35(+)(13)} & \text{72.99$\pm$0.69(+)(13)} & \text{46.25$\pm$0.34(+)(13)} & 4/0/0 & 13.75 & RL \\ \hline
arch2vec-BO & \text{94.45$\pm$0.13(+)(15)} & \text{91.55$\pm$0.02(+)(4)} & \text{73.47$\pm$0.19($\approx$)(3)} & \text{46.38$\pm$0.03(+)(11)} & 3/1/0 & 8.25 & BO \\ \hline
NPENAS-NP & \text{94.66$\pm$0.14(+)(7)} & \underline{91.58$\pm$0.11(+)(2)} & \text{73.20$\pm$0.59(+)(11)} & \text{46.28$\pm$0.39(+)(12)} & 4/0/0 & 8.00 & EA \\ \hline
NPENAS-SSRL & \underline{94.75$\pm$0.21(+)(2)} & \text{91.53$\pm$0.16(+)(7)} & \text{73.35$\pm$0.40(+)(5)} & \text{46.53$\pm$0.24(+)(5)} & 4/0/0 & 4.75 & EA \\ \hline
NPENAS-SSCCL & \text{94.71$\pm$0.18(+)(4)} & \text{91.55$\pm$0.18(+)(5)} & \underline{73.48$\pm$0.07(+)(2)} & \text{46.59$\pm$0.19(+)(3)} & 4/0/0 & 3.50 & EA \\ \hline
SAENAS-NE & \text{94.61$\pm$0.10(+)(11)} & \text{91.57$\pm$0.14(+)(3)} & \text{73.35$\pm$0.38(+)(7)} & \underline{46.61$\pm$0.11($\approx$)(2)} & 3/1/0 & 5.75 & EA \\ \hline
CAP & \text{94.65$\pm$0.20(+)(8)} & \text{91.54$\pm$0.10(+)(6)} & \text{73.41$\pm$0.17(+)(4)} & \text{46.47$\pm$0.07(+)(8)} & 4/0/0 & 6.50 & EA \\ \hline
DCL-ENAS & \textbf{95.00$\pm$0.10(1)} & \textbf{91.61$\pm$0.00(1)} & \textbf{73.48$\pm$0.06(1)} & \textbf{46.66$\pm$0.13(1)} & NA & 1.00 & EA \\ \hline
\end{tabular}

%% file: bench-test.tex
\begin{tabular}{|c|c|c|c|c|c|c|c|}
\hline Benchmark & NASBench-101 & NASBench-201-CIFAR10-valid & NASBench-201-CIFAR100 & NASBench-201-ImageNet16-120 & +/$\approx$/- & Average Rank & Optimized type \\ \hline
Budget & 150 & 100 & 100 & 100 & - & - & - \\ \hline
Random & \text{93.58$\pm$0.21(+)(19)} & \text{93.75$\pm$0.19(+)(19)} & \text{71.38$\pm$0.66(+)(19)} & \text{45.35$\pm$0.63(+)(18)} & 4/0/0 & 18.75 & R \\ \hline
REA & \text{93.67$\pm$0.23(+)(17)} & \text{94.10$\pm$0.29(+)(14)} & \text{72.82$\pm$0.87(+)(14)} & \text{46.11$\pm$0.56(+)(16)} & 4/0/0 & 15.25 & EA \\ \hline
BANANAS-PE & \text{94.12$\pm$0.16(+)(6)} & \text{94.05$\pm$0.37(+)(15)} & \text{72.41$\pm$1.00(+)(17)} & \text{46.17$\pm$0.52($\approx$)(14)} & 3/1/0 & 13.00 & BO \\ \hline
BANANAS-AE & \text{94.10$\pm$0.19(+)(9)} & \text{94.26$\pm$0.23(+)(10)} & \text{73.16$\pm$0.75(+)(10)} & \text{46.44$\pm$0.22($\approx$)(4)} & 3/1/0 & 8.25 & BO \\ \hline
BANANAS-PAPE & \text{94.12$\pm$0.15(+)(5)} & \text{94.23$\pm$0.28(+)(12)} & \text{73.20$\pm$0.65(+)(9)} & \text{46.30$\pm$0.35($\approx$)(10)} & 3/1/0 & 9.00 & BO \\ \hline
BOHAMIANN & \text{93.94$\pm$0.20(+)(13)} & \text{94.03$\pm$0.25(+)(17)} & \text{72.62$\pm$0.72(+)(15)} & \text{46.06$\pm$0.52(+)(17)} & 4/0/0 & 15.50 & BO \\ \hline
BOGCN-NAS & \text{94.05$\pm$0.21(+)(11)} & \text{94.27$\pm$0.19(+)(9)} & \text{73.27$\pm$0.42(+)(8)} & \text{46.35$\pm$0.39($\approx$)(8)} & 3/1/0 & 9.00 & BO \\ \hline
DNGO & \text{93.89$\pm$0.27(+)(14)} & \text{94.05$\pm$0.23(+)(16)} & \text{72.53$\pm$0.67(+)(16)} & \text{46.21$\pm$0.47(+)(12)} & 4/0/0 & 14.50 & BO \\ \hline
GCN-predictor & \text{93.64$\pm$0.19(+)(18)} & \text{93.80$\pm$0.27(+)(18)} & \text{71.52$\pm$0.73(+)(18)} & \text{45.29$\pm$0.70(+)(19)} & 4/0/0 & 18.25 & R \\ \hline
GP-BO & \text{94.11$\pm$0.17(+)(7)} & \text{94.21$\pm$0.29(+)(13)} & \text{73.35$\pm$0.49(+)(5)} & \text{46.32$\pm$0.39($\approx$)(9)} & 3/1/0 & 8.50 & BO \\ \hline
local-search & \text{94.00$\pm$0.23(+)(12)} & \text{94.28$\pm$0.22(+)(7)} & \text{72.98$\pm$0.79(+)(13)} & \text{46.36$\pm$0.28(+)(7)} & 4/0/0 & 9.75 & - \\ \hline
arch2vec-RL & \text{93.76$\pm$0.16(+)(16)} & \text{94.25$\pm$0.31(+)(11)} & \text{73.03$\pm$0.71(+)(12)} & \text{46.15$\pm$0.40(+)(15)} & 4/0/0 & 13.50 & RL \\ \hline
arch2vec-BO & \text{93.84$\pm$0.17(+)(15)} & \underline{94.36$\pm$0.02(+)(2)} & \text{73.48$\pm$0.13($\approx$)(3)} & \textbf{46.55$\pm$0.23($\approx$)(1)} & 2/2/0 & 5.25 & BO \\ \hline
NPENAS-NP & \text{94.11$\pm$0.16(+)(8)} & \text{94.33$\pm$0.15(+)(5)} & \text{73.16$\pm$0.63(+)(11)} & \text{46.23$\pm$0.45($\approx$)(11)} & 3/1/0 & 8.75 & EA \\ \hline
NPENAS-SSRL & \text{94.13$\pm$0.15($\approx$)(4)} & \text{94.27$\pm$0.22(+)(8)} & \text{73.34$\pm$0.46(+)(6)} & \text{46.43$\pm$0.24($\approx$)(5)} & 2/2/0 & 5.75 & EA \\ \hline
NPENAS-SSCCL & \text{94.15$\pm$0.15($\approx$)(3)} & \text{94.30$\pm$0.22(+)(6)} & \underline{73.49$\pm$0.07($\approx$)(2)} & \text{46.19$\pm$0.48($\approx$)(13)} & 1/3/0 & 6.00 & EA \\ \hline
SAENAS-NE & \text{94.08$\pm$0.16(+)(10)} & \text{94.33$\pm$0.15(+)(4)} & \text{73.30$\pm$0.53(+)(7)} & \text{46.42$\pm$0.25($\approx$)(6)} & 3/1/0 & 6.75 & EA \\ \hline
CAP & \underline{94.18$\pm$0.16($\approx$)(2)} & \text{94.34$\pm$0.06(+)(3)} & \text{73.41$\pm$0.22($\approx$)(4)} & \text{46.44$\pm$0.36($\approx$)(3)} & 1/3/0 & 3.00 & EA \\ \hline
DCL-ENAS& \textbf{94.20$\pm$0.05(1)} & \textbf{94.37$\pm$0.00(1)} & \textbf{73.49$\pm$0.06(1)} & \underline{46.45$\pm$0.15(2)} & NA & 1.25 & EA \\ \hline
\end{tabular}